\documentclass{article}

\usepackage{microtype}
\usepackage{graphicx}
\usepackage{subcaption}
\usepackage{booktabs} 

\usepackage{hyperref}

\usepackage{dsfont}

\usepackage[accepted]{icml2026}

\usepackage{amsmath}
\usepackage{amssymb}
\usepackage{mathtools}
\usepackage{amsthm}

\input c.symbols


\usepackage[capitalize,noabbrev]{cleveref}

\theoremstyle{plain}
\newtheorem{theorem}{Theorem}[section]

\theoremstyle{definition}

\theoremstyle{remark}

\usepackage[textsize=tiny]{todonotes}

\icmltitlerunning{Hierarchical Successor Representation}

\begin{document}

\twocolumn[
  \icmltitle{Hierarchical Successor Representation for Robust Transfer}



  \icmlsetsymbol{equal}{*}

  \begin{icmlauthorlist}
    \icmlauthor{Changmin Yu}{cbl}
    \icmlauthor{Máté Lengyel}{cbl,ceu}
  \end{icmlauthorlist}

  \icmlaffiliation{cbl}{Computational and Biological Learning Lab, Department of Engineering, University of Cambridge, Cambridge, United Kingdom}
  \icmlaffiliation{ceu}{Department of Cognitive Science, Central European University, Budapest, Hungary}

  \icmlcorrespondingauthor{Changmin Yu}{changmin.yu98@gmail.com}

  \icmlkeywords{Machine Learning, ICML}

  \vskip 0.3in
]



\printAffiliationsAndNotice{}  

\begin{abstract}
  The successor representation (SR) provides a powerful framework for decoupling
  predictive dynamics from rewards, enabling rapid generalisation across reward
  configurations. However, the classical SR is limited by its inherent policy
  dependence: policies change due to ongoing learning, environmental
  non-stationarities, and changes in task demands, making established predictive
  representations obsolete. Furthermore, in topologically complex environments,
  SRs suffer from spectral diffusion, leading to dense and overlapping features
  that scale poorly. Here we propose the Hierarchical Successor Representation
  (HSR) for overcoming these limitations. By incorporating temporal abstractions
  into the construction of predictive representations, HSR learns stable state
  features which are robust to task-induced policy changes. Applying
  non-negative matrix factorisation (NMF) to the HSR yields a sparse, low-rank
  state representation that facilitates highly sample-efficient transfer to
  novel tasks in multi-compartmental environments. Further analysis reveals that
  HSR-NMF discovers interpretable topological structures, providing a
  policy-agnostic hierarchical map that effectively bridges model-free
  optimality and model-based flexibility. Beyond providing a useful basis for
  task-transfer, we show that HSR's temporally extended predictive structure can
  also be leveraged to drive efficient exploration, effectively scaling to
  large, procedurally generated environments.
\end{abstract}

\vspace{-20pt}
\section{Introduction}
\label{sec: intro}
Despite significant progress in reinforcement learning (RL) over the past two
decades~\cite{sutton1998reinforcement, matsuo2022deep}, achieving robust and
efficient generalisation remains a major challenge~\cite{cobbe2019quantifying}.
For instance, an agent navigating in a maze may encounter changes in goal
locations or reward configurations while the environmental geometry remains
unchanged. In such settings, the key to efficient learning is to construct
robust, reusable state representations that capture persistent structure of the
environment~\citep{zhang2020learning, agarwal2021contrastive}, allowing for
rapid adaptation when goals change. However, achieving robust transfer following
task changes is difficult since the policy that is optimal for one task may be
suboptimal for another, and associated policy adaptation could invalidate
representations learned under the previous behaviour.

The successor representation (SR) framework offers a promising avenue for such
transfer~\citep{dayan1993improving}. By decomposing value function into a linear
combination of expected state occupancy and state-dependent reward function, the
SR enables generalisation through rapid policy re-evaluation when the reward
function changes. However, two bottlenecks limit the practical scope of SRs in
transfer-oriented regimes. First, SRs are inherently policy dependent, since the
predictive occupancy changes when the policy is updated. This issue becomes
acute when generalisation requires not only re-evaluating state values but also
re-optimising behaviour. Second, SR features are dominated by diffusion, leading
to globally supported, overlapping features, a phenomenon particularly evident
in spectral components of SRs~\cite{mahadevan2007proto,
stachenfeld2017hippocampus}. The resultant state features lack interpretability
and scale poorly as the environment becomes larger or more topologically complex
(e.g., compartmentalised; Figure~\ref{fig: four_room_motivation}b).
Consequently, it is preferable to construct a representation that highlights
local transition structure rather than spreading its mass across distant regions
of the state space.

We propose the Hierarchical Successor Representation (HSR) to address these
limitations. The central idea is to extend the classical SR formulation by
incorporating temporal abstractions -- formalised through the ``options''
framework~\cite{sutton1999between}. Options define interpretable, temporally
extended courses of actions, with distinct policies, initiation and terminal
conditions. By reasoning at this higher level of temporal abstraction, HSR
aggregates dynamics over interpretable behavioural segments rather than
single-step actions, yielding a predictive representation that is less sensitive
to task-induced variations in low-level control. Intuitively, whilst optimal
primitive actions may vary across tasks, the high-level strategy (e.g., go
through bottleneck states into the next room) remains stable. The HSR leverages
this stability to build a task-agnostic representation, enabling more efficient
generalisation without requiring expensive re-learning of the state feature
maps.

To further promote interpretability and scalability, we seek a compact low-rank
basis of these predictive representations. While existing spectral approaches
use eigenvectors of transition structure to construct geometry-aware
bases~\cite{mahadevan2007proto, stachenfeld2017hippocampus}, the resultant
components are typically globally supported and sign-indefinite. We therefore
employ non-negative matrix factorisation to obtain sparse, localised, and
interpretable factors~\cite{lee1999learning}. In multi-compartmental domains, we
show that these bases align naturally with topological features such as
bottleneck states, and yield a policy-agnostic multi-scale map that can be
reused across tasks without sacrificing optimality. These features hence
facilitate sample-efficient transfer across reward configurations. Beyond
generalisation, we find that HSR can be used to guide intrinsically motivated
exploration in sparse-reward environments~\citep{schmidhuber1991curious,
pathak2017curiosity}, and enables efficient scaling to larger, procedurally
generated maze environments where existing exploration strategies become
prohibitively expensive. Importantly, the utility of the proposed HSR framework
remains robust with respect to environmental geometry, option definition, and
the existence of pre-constructed option set. Collectively, our results suggest
that integrating temporal abstraction into predictive representations leads to a
robust, interpretable state representation, and provides a practical bridge
between model-free efficiency and model-based
flexibility~\cite{gershman2018successor}: HSR retains the computational
advantages of SR-style predictive representations whilst producing hierarchical,
interpretable structure that supports robust transfer.
\vspace{-20pt}

\begin{figure*}[ht]
  \begin{center}
    \centerline{\includegraphics[width=.98\linewidth]{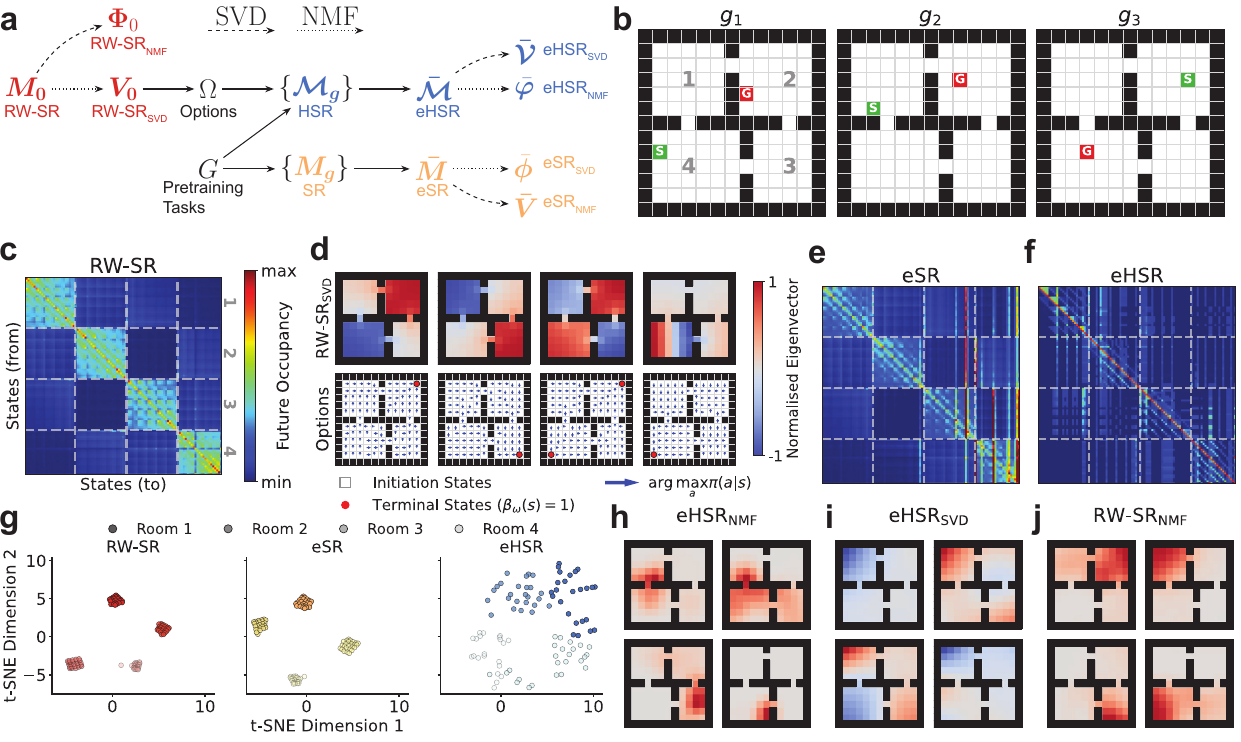}}
    \vspace{-5pt}
    \caption{
      \textbf{Temporal abstraction yields hierarchical successor representation.} 
      \textbf{a.} Schematic of the computational process underlying the
      construction of hierarchical successor representations, and corresponding
      low-dimensional basis through NMF($\vec \Phi$) and singular value
      decomposition (SVD; $\vec V$). 
      \textbf{b.} Exemplar pretraining regimes ($G$).
      \textbf{c.} SR matrix corresponding to the random-walk policy (RW-SR).
      Note that state indices were permuted to respect the topological structure
      of the four-room environment (indicated by gray dashed lines).
      \textbf{d.} Principal eigenvectors (ranked by eigenvalues) of the RW-SR
      matrix (top), as well as their corresponding eigenoption-specific policies
      (bottom). 
      \textbf{e.} Expected SR matrix (eSR; $\bar{\vec M}$).
      \textbf{f.} Expected HSR matrix (eHSR; $\bar{\vec \mathcal{M}}$).
      \textbf{g.} Two-dimensional t-SNE projections~\cite{maaten2008visualizing}
      of the row-space of RW-SR (left), eSR (middle), and eHSR (right).
      \textbf{h.} Principal NMF basis vectors (ranked by basis norm) of the eHSR
      matrix.
      \textbf{i.} Principal eigenvectors of the expected eSR matrix.
      \textbf{j.} Same as \textbf{j}, but for the RW-SR matrix.
      Note that all presented basis vectors were normalised by their
      corresponding maximum absolute values, hence leaving their signs
      invariant.
    }
    \label{fig: four_room_motivation}
  \end{center}
  \vspace{-25pt}
\end{figure*}

\section{Preliminaries}
\label{sec: background}
\textbf{Reinforcement Learning.}
We consider the classical RL setup in discrete Markov Decision Processes (MDPs;
\citealt{sutton1998reinforcement}), defined by the tuple, $\langle \mathcal{S},
\mathcal{A}, \mathcal{P}, \mathcal{R}, \gamma\rangle$, where $\mathcal{S}$ and
$\mathcal{A}$ denote the state (of size $N$) and action spaces (of size $M$),
respectively, $\mathcal{P}: \mathcal{S}\times\mathcal{A} \rightarrow
\Delta\left(\mathcal{S}\right)$ is the state transition probability
($\Delta(\mathcal{S})$ denotes a probability measure over $\mathcal{S}$),
$\mathcal{R}: \mathcal{S}\rightarrow\mathbb{R}$ is the reward function, and
$\gamma\in[0, 1]$ is the temporal discounting factor. The goal of an RL
algorithm is to learn a policy, $\pi: \mathcal{S}\times\mathcal{A}\rightarrow
[0, 1]$, such that the agent achieves maximum cumulative future reward starting
from any state. Formally, the policy-dependent value and action-value functions
are defined as the expected discounted future return, $V^\pi(s) =
\mathbb{E}_\pi\left[\sum_{t=0}^{\infty}\gamma^t \mathcal{R}(s_t)|s_0=s\right]$.


\textbf{Successor Representation.}
The SR framework enables decomposition of the (action) value function into a
linear combination of time-collapsed predictive representation and
state-dependent reward function~\cite{dayan1993improving}.
\begin{equation}
  V^\pi(s) = \sum_{s'\in\mathcal{S}}
  \mathbb{E}_\pi\left[\sum_{t=0}^{\infty}\gamma^t \mathds{1}(s_t, s')|s_0=s\right]
  \mathcal{R}(s')
  = \vec M^\pi\cdot \vec R\,,
  \label{eq: sr_def}
\end{equation}
It is clear from the above decomposition that the SR formulation enables rapid
generalisation across varying reward configurations, given efficient one-step
linear update for policy re-evaluation under new reward functions. However, it
is often necessary to simultaneously adapt the policy for achieving optimal
performance under new tasks. Drawing upon similar recursive definitions of value
functions and the SR, it is hence possible to learn the SR matrix online with
temporal-difference (TD) learning~\cite{dayan1993improving}. Specifically, given
the state-transition tuple $(s_t, a_t, s_{t+1})$, and learning rate
$\alpha\in\mathbb{R}$, the update is as following.
\begin{equation}
  \hspace{-5pt}\hat{\vec M}_{s_ts'} \leftarrow \hat{\vec M}_{s_ts'} + \alpha\left(
    \mathds{1}(s_{t}, s') + \gamma\hat{\vec M}_{s_{t+1}s'} - \hat{\vec M}_{s_ts'}
  \right)\,,
  \label{eq: sr_td}
\end{equation}

\textbf{Hierarchical Reinforcement Learning.} Under the hierarchical RL
setup, the ``option'' framework enables the incorporation of temporal
abstraction in RL: instead of interacting with the environment at every time
step, the agent is able to engage in temporally extended, interpretable action
sequences~\cite{sutton1999between}. Formally, an option $\omega_i \in \Omega$ is
defined by the tuple, $\langle \mathcal{I}_i, \pi_i, \beta_i\rangle$, where
$\mathcal{I}_i\subseteq \mathcal{S}$ denotes the initiation set such that an
option $\omega_i$ is available in state $s$ if and only if $s \in
\mathcal{I}_i$, $\pi_i: \mathcal{S}\times\mathcal{A}\rightarrow [0, 1]$ denotes
the option-specific policy, and $\beta_i: \mathcal{S}\rightarrow [0, 1]$ denotes
the option's termination function. Under the option framework, the agent learns
a high-level policy, $\mu: \mathcal{S}\times\bar{\mathcal{A}} \rightarrow [0,
1]$, where $\bar{\mathcal{A}} = \mathcal{A}\cup\Omega$ denotes the extended
action space. Options can be either explicitly pre-defined given expert
knowledge~\cite{sutton1999between, dietterich2000hierarchical, wang2021cpg} or
automatically discovered~\cite{mcgovern2001automatic, bacon2017option,
machado2017laplacian, jinnai2019discovering}. Here we focus on the
``eigenoption'' framework~\citep{machado2017laplacian}, a task-agnostic
option-discovery approach that exploits the transition structure of the
environment. Specifically, given the random-walk SR matrix, $\vec M_0 \in [0,
1]^{N\times N}$, each eigenvector $\vec v_i \in \mathbb{R}^{N}$, such that $\vec
M_0\vec v_i = \lambda_i\vec v_i$, defines an option $\omega_i$. The
option-specific policy is obtained by solving the MDP with pseudo-reward,
defined as following.
\begin{equation}
  r_i(s, s') = \vec v_{i, s'} - \vec v_{i, s}\,, \forall s, s' \text{ where } \exists a \text{ s.t. } \mathcal{P}(s'|s, a) > 0\,,
  \label{eq: eigenoption_reward}
\end{equation}
Upon learning converges with respect to the pseudo-reward function, the optimal
pseudo value function, $q^\ast_i(s, a)$, for all $s$ and $a$, defines the
option-specific policy: $\pi_i(a|s) = \mathds{1}(a, \argmax_{a'}q^\ast_i(s,
a'))$. The termination function is a binary function such that $\beta_i(s) = 1$
if $q^\ast_i(s, a') \leq 0$ $\forall a'\in\mathcal{A}$, and $0$ otherwise (i.e.,
local maximum of the pseudo-reward function). The initiation set is defined as
the complement to the set of terminal states, $\mathcal{I}_i = \mathcal{S}
\backslash \{s | \beta_i(s) = 1\}$. In practice, options are defined by the set
of principal eigenvectors of the random-walk SR matrix, which capture the global
smoothness of the environment~\cite{shi2000normalized}, with decreasing
timescales of diffusion (with eigenvalues) in the classical four-room
environment (Figure~\ref{fig: four_room_motivation}d).
\vspace{-5pt}

\section{Hierarchical Successor Representation}
\label{sec: hsr_method}
\vspace{-3pt}
Despite the SR framework provides a predictive state representation that enables
rapid transfer in response to reward changes, reaching optimal transferability
still necessitates updating the SR corresponding to policy changes in a
model-free fashion (Equation~\ref{eq: sr_td}). To achieve robust generalisation,
we seek a state representation that could retain the predictive nature of the
classical SR, but of weaker policy-dependence. We extend the SR formulation via
incorporating temporally extended options, leading to the Hierarchical SR (HSR),
defined as the predictive state occupancy distribution with respect to the
high-level policy, $\mu$ (Equation~\ref{eq: sr_def}).
\begin{equation}
    \mathcal{M}^\mu_{ss'} \equiv 
    \mathbb{E}_{\mu}\left[
      \vec M^{\bar{a}}_{ss'} + \mathbb{E}_{\bar{a}}\left[
        \gamma^{\tau_{s\bar{a}}}\mathcal{M}^\mu_{s_{\tau_{s\bar{a}}}s'}
      \right]
    \right]\,,
  \label{eq: hsr_def}
\end{equation}
where $\tau_{s\bar{a}} = \mathbb{E}_{\mathcal{P},
\pi_{\bar{a}}}\left[\sum_{t=0}^{\infty} \mathds{1}(\beta_{\bar{a}}(s_t),
0)|s_0=s\right]$ denotes the expected duration following action $\bar{a}$, and
$\vec M^{\bar{a}}$ denotes the action-specific SR matrix, corresponding to the
policy $\pi_{\bar{a}}$, which is defined to be the single-action policy if
$\bar{a}\in\mathcal{A}$, and option-specific policy if $\bar{a}\in\Omega$. We
have overloaded the notation such that each primitive action is treated as a
one-step pseudo-option, such that given any $(s,
a)\in\mathcal{S}\times\mathcal{A}$, we have $\mathcal{I}_{sa} = \{s\}$,
$\pi_{sa}(a'|s) = \mathds{1}(a, a')$, and $\beta_{sa}(s') = 1$ if and only if
$\mathcal{P}(s'|s, a) > 0$ and $0$ otherwise. The recursive definition yields
the following HSR Bellman operator, $\mathcal{T}$ (full derivations can be found
in \ref{sec: hsr_derivations}).
\begin{equation}
  \begin{aligned}
    &(\mathcal{T}^\mu\mathcal{M})_{ss'} = \mathcal{B}^\mu_{ss'} + \sum_{\tilde{s}\in\mathcal{S}}\mathcal{G}^\mu_{s\tilde{s}}\mathcal{M}_{\tilde{s}s'}\,, \text{ where }\\
    &\mathcal{B}^\mu_{ss'} = \sum_{\bar{a}\in\bar{\mathcal{A}}}\mu(\bar{a}|s)\vec M^{\bar{a}}_{ss'}\,, \quad
    \mathcal{G}^\mu_{ss'} = \sum_{\bar{a}\in\bar{\mathcal{A}}}\mu(\bar{a}|s)\vec F^{\bar{a}}_{ss'}\,,
  \end{aligned}
  \label{eq: hsr_bellman}
\end{equation}
where $\vec{F}^{\bar{a}}_{s_j s_j} =
\mathbb{E}\left[\gamma^{\tau_{s_i\bar{a}}}\mathds{1}(s_{\tau_{s_i\bar{a}}}, s_j)
| s_0=s_i\right]$ is the discounted termination kernel for (pseudo) option
$\bar{a}$. With a bit of algebra (Appendix~\ref{sec: hsr_derivations}), we can
show that $\vec F^{\bar{a}}$ can be computed analytically as following.
\begin{equation}
  \vec F^{\bar{a}} = \vec M^{\bar{a}}\text{diag}(\vec\beta_{\bar{a}})\,,
\end{equation}
where $\text{diag}(\vec \beta_{\bar{a}})$ is the diagonal matrix with the $i$-th
diagonal element equals $\beta_{\bar{a}}(s_i)$. Similar to the standard Bellman
operator, $\mathcal{T}$ is a contraction mapping.
\begin{theorem}[Contraction of HSR Bellman Operator]
  Let $\mathcal{T}$ be the HSR Bellman operator defined by Equation~\ref{eq:
  hsr_bellman}. For any discount factor $\gamma < 1$ and option durations $\tau
  \geq 1$, $\mathcal{T}$ is a contraction mapping with respect to the max-norm
  (proof can be found in \ref{sec: bellman_contraction_proof}).
  \begin{equation}
    || \mathcal{T}^{\mu} \mathcal{M} - \mathcal{T}^{\mu} \mathcal{M}' ||_{\infty} \leq 
    \gamma || \mathcal{M} - \mathcal{M}' ||_{\infty}
\end{equation}
for any $\mathcal{M}$ and $\mathcal{M}'$.
\label{theorem: hsr_contraction}
\vspace{-5pt}
\end{theorem}
Consequently, the HSR Bellman operator defines an equivalent iterative update
rule for learning the HSR online, similar to the TD-learning for standard SR
(Equation~\ref{eq: sr_td}): $\hat{\mathcal{M}}_{s_ts'} \leftarrow
\hat{\mathcal{M}}_{s_ts'} + \alpha\delta^{\mathcal{M}}_t$. Given the
state-transitions trajectory following the (pseudo) option, $\bar{a}_t$,
$\left(\{(s_{t+k}, \bar{a}_{t+k})\}_{k=0}^{\tau_t},
s_{t+\tau_t+1}\right)$\footnote{Note that we have overloaded notations to denote
$\tau_t := \tau_{s_t\bar{a}_t}$.}, the TD-error defined as following.
\begin{equation}
  \delta^{\mathcal{M}}_t = \left(
    \sum_{k=0}^{\tau_t}\gamma^k \mathds{1}(s_{t+k}, s') + 
    \gamma^{\tau_t+1}\hat{\mathcal{M}}_{s_{\tau_t+1}s'} - \hat{\mathcal{M}}_{s_ts'}
  \right)
\label{eq: hsr_td}
\end{equation}

While the HSR inherently supports online TD learning to track a changing policy,
this adaptation can be sensitive to transient behavioural shift. Therefore, in
transfer scenarios where pre-exposure to the environment is possible via a
distribution of pre-training tasks (Figure~\ref{fig: four_room_motivation}b), we
advocate for a stable offline construction. Specifically, we compute the
\textit{Expected} HSR (eHSR) by averaging the HSRs derived from optimal policies
in these pre-training tasks, yielding a generalised basis that further relaxes
the policy dependence in the constructed predictive representation. 

\textbf{Low-Rank Basis of HSR.} Constructing low-dimensional state
representations via the decomposition of transition dynamics has long served as
an effective representation learning approach in RL. Seminal works, such as
Proto-value functions (PVFs;~\citealt{mahadevan2007proto}) and spectral successor
features~\cite{stachenfeld2017hippocampus}, rely on eigenvectors of the
Laplacian operator and SR, respectively, to define a geometric basis. Following
this paradigm, we also consider using eigenvectors of the HSR as state
representations, which indeed captures the globally supported smooth dynamics
(Figure~\ref{fig: four_room_motivation}i). However, the HSR introduces a
temporally extended, topology-aware predictive representation, which is
fundamentally different from the classical SR. Classical SRs exhibit topological
collapse: the predictive features of intra-compartment states are tightly
clustered, reducing the effective dimension within each topological region to a
single point mass (Figure~\ref{fig: four_room_motivation}g). This implies the
classical SR is numerically low-rank, a regime where singular value
decomposition (SVD) is optimal for capturing the global modes of variation (also
indicated by the eigenvalue distribution; Figure~\ref{fig:
hsr_spectral_supplementary}a). 

In contrast, HSR dynamics is piecewise-smooth, characterised by manifolds with
effective intra-compartment dispersion, whilst preserving the inter-compartment
differences (Figure~\ref{fig: four_room_motivation}g). Whilst SVD can compress
the HSR, its orthogonal basis vectors inevitably ``smear'' the multi-scale
structure of these manifolds across global components, leading to oscillatory
ringing artefacts induced by orthogonal low-rank reconstruction (also known as
the Gibbs phenomenon; Figure~\ref{fig: four_room_motivation}i;
\citealt{stephane1999wavelet}) and heavier-tailed spectral components
(Figure~\ref{fig: hsr_spectral_supplementary}a). Consequently, we employ
non-negative matrix factorisation (NMF) to extract the $K$-dimensional sparse,
localised basis~\cite{lee1999learning, hoyer2004non}.
\begin{equation}
  \mathcal{M} = \vec \varphi\cdot \vec H\,, 
  \text{ where } \vec\varphi \in\mathbb{R}^{N\times K}, 
  \vec H \in\mathbb{R}^{K\times N} \text{ s.t. } \vec \varphi, \vec H\geq 0\,,
\end{equation}

The locally dispersed geometry of the HSR enables the NMF to discover the
ground-truth underlying generative parts (Figure~\ref{fig:
four_room_motivation}h; \citealt{donoho2003does}). However, the feature collapse
in standard SRs makes ``parts-based'' decomposition via NMF ill-posed, due to
the fact that there is insufficient intra-compartment variance to define
distinct basis components, leading to trivial and overlapping factorisations
(Figure~\ref{fig: four_room_motivation}j; Figure~\ref{fig:
hsr_basis_supplementary}b). Schematics of the full computational process for
deriving HSR and its low-rank decomposition can be found in Figure~\ref{fig:
four_room_motivation}a, and corresponding pseudocode can be found in
Algorithm~\ref{alg: hsr_nmf}.
\vspace{-5pt}

\begin{figure*}[ht]
  \begin{center}
    \centerline{\includegraphics[width=.98\textwidth]{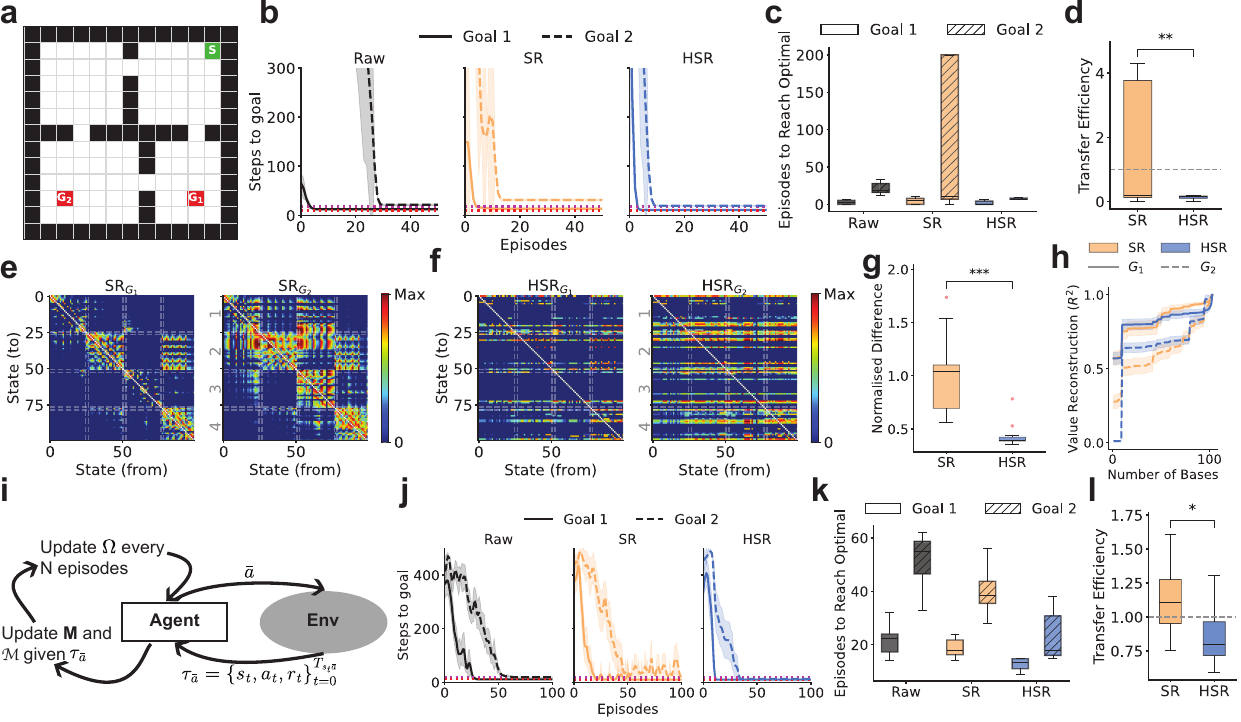}}
    \vspace{-5pt}
    \caption{
    \textbf{HSR provides a stable state representation and enables sample-efficient transfer across tasks with shared transition dynamics.}
      \textbf{a.} Schematics of the four-room environment. All agents were
      firstly trained to reach $G_1$, and subsequently transferred to the new
      task with goal location $G_2$, both from a fixed start location. 
      \textbf{b.} Number of steps to reach the goal location (mean $\pm$ s.e.)
      for Q-learning agents with linear function approximation, given different
      state representations. All state representations (apart from one-hot
      representation) were simultaneously trained with the value function.
      Dashed horizontal red and magenta lines indicate optimal number of steps
      to reach $G_1$ ($10$) and $G_2$ ($18$) from the shared start state.
      \textbf{c.} Number of training episodes to reach optimal performance in
      the $G_1$ and $G_2$ tasks, for all agents in \textbf{b} (number set of
      $200$ if agents fail to reach optimal performance within $200$ episodes).
      \textbf{d.} Transfer efficiency (normalised ratio between number of
      training episodes to reach optimal performance in $G_1$- and $G_2$-tasks)
      for SR- and HSR-based agents (two-sided two-sample t-test; $p = 0.008$,
      $df=38$).
      \textbf{e.} SR matrices (log-scale for visual clarity) after the
      corresponding agent were trained to reach optimal performance in $G_1$
      (left) and $G_2$ (right) tasks. Note that rows and columns of SR matrices
      are permuted to restore the local topological structure of the environment
      (Figure~\ref{fig: four_room_motivation}b). We omit showing diagonal
      elements for visual clarity. 
      \textbf{f.} Same as \textbf{e}, but for HSR matrices. 
      \textbf{g.} Degrees of change in predictive representation
      $\left(\frac{||M_1 - M_2||^2_F}{||M_1||^2_F}\right)$ after agents were
      trained in $G_1$ and $G_2$ tasks, for SR and HSR matrices, respectively
      (two-sided two-sample t-test; $p < 0.001$, $\text{df} = 38$). 
      \textbf{h.} Reconstruction $R^2$ scores of ground-truth optimal value
      functions (computed via dynamic programming) for $G_1$ and $G_2$ tasks
      given varying number of SR/HSR basis after agents are trained to reach
      optimal performance in $G_1$ tasks.
      \textbf{i.} Schematics of the HSR framework with online-constructed option
      set. 
      \textbf{j.} Same as \textbf{b}, but for agents with online-constructed
      option sets.
      \textbf{k.} Same as \textbf{c}, but for agents in \textbf{j}.
      \textbf{l.} Same as \textbf{d}, but for SR and HSR agents with online
      constructed options (two-sided two-sample t-test; $p = 0.026$, $df=38$.). 
    }
    \label{fig: four_room_row_features_main}
  \end{center}
  \vspace{-25pt}
\end{figure*}

\section{Results}
\label{sec: results}
We evaluate the efficacy of the proposed HSR-based state representations in
facilitating robust transfer, interpretable state abstractions, and scalable
exploration. The empirical experiments aim to test three central hypotheses: (a)
HSR features are more robust to task-induced policy changes than standard SRs
under transfer-learning scenarios; (b) the topological structure of HSR is
uniquely amenable to NMF decomposition, yielding interpretable state features
that facilitates stronger generalisation than existing spectral-based
decomposition methods; and (c) these temporally extended predictive features
enable efficient and scalable exploration in large, procedurally generated maze
environments. All agents in transfer experiments were implemented based on
Q-learning with linear function approximation. Unless otherwise stated,
empirical evaluations are based on $20$ random seeds. Implementation details can
be found in Appendix~\ref{sec: implementation_detail}.

\begin{figure*}[ht]
  \begin{center}
    \centerline{\includegraphics[width=\textwidth]{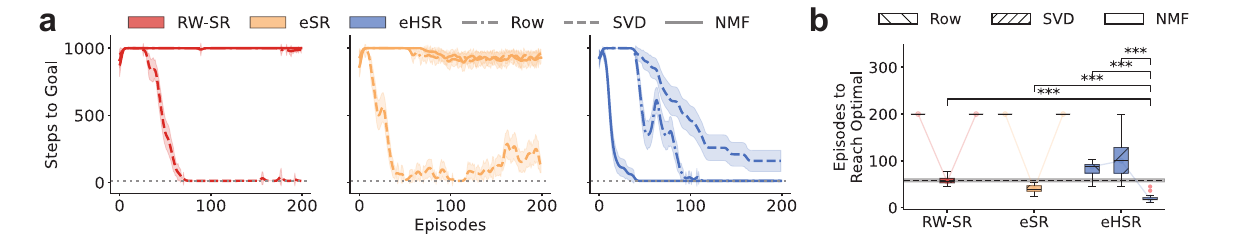}}
    \vspace{-5pt}
    \caption{
      \textbf{NMF basis of HSR supports sample-efficient transfer.}
      \textbf{a.} Training curves (left) and number of training episodes to
      reach optimal performance (right) in $G_1$ tasks (Figure~\ref{fig:
      four_room_row_features_main}a) for Q-learning agents with linear function
      approximation, given different low-dimensional basis as state
      representations. All agents were assumed to have received necessary
      pretraining for constructing base matrices (SR/HSR) before corresponding
      low-dimensional basis were extracted. Gray dotted line indicates optimal
      number of steps to reach $G_1$.
      \textbf{b.} Number of training episodes required to reach optimal
      performance for all agents in \textbf{a} (two-sided Wilcoxon signed-rank
      test; $\text{RW-SR}_{\text{SVD}}$ vs $\text{HSR}_{\text{NMF}}$: $p =
      7.91\times 10^{-17}$; $\text{ESR}_{\text{SVD}}$ vs
      $\text{HSR}_{\text{NMF}}$: $p = 2.93\times 10^{-9}$;
      $\text{HSR}_{\text{SVD}}$ vs $\text{HSR}_{\text{NMF}}$: $p =
      6.73\times10^{-10}$; $\text{HSR}_{\text{NMF}}$ vs
      $\text{HSR}_{\text{Row}}$: $p = 4.59\times10^{-18}$; $N = 20$ for all
      tests). Gray dashed horizontal line indicates the same number for the
      baseline agent with one-hot state encoding (shaded area indicates s.e.). 
      }
    \label{fig: four_room_nmf_svd_main}
  \end{center}
  \vspace{-25pt}
\end{figure*}

\vspace{-2pt}
\subsection{HSR facilitates robust few-shot transfer.}
\vspace{-2pt}
We first investigate the robustness of HSR under changing reward configurations
in the classical four-room environment (Figure~\ref{fig:
four_room_row_features_main}a). Here we assume both SRs and HSRs are learned
online as policy changes, both within and between tasks. To demonstrate the
utility of canonical HSR features, we take rows of SR and HSR matrices as state
features. We also implement an agent with fixed one-hot state encoding (``Raw'')
as the zero-transfer baseline. For fair comparison with the HSR-based agents,
all agents were trained to choose actions from the augmented action space,
consists of primitive actions and the set of $8$ principal eigenoptions
(Figure~\ref{fig: hsr_basis_supplementary}a; note that standard SR matrices are
still updated on a stepwise basis during the course of executing options).
Agents were trained to navigate to a specific terminal goal state ($G_1$) until
convergence, and subsequently transferred to navigate to a new goal ($G_2$)
starting from the same initial location\footnote{Goal locations are selected
such that they do not coincide with any of the terminal states of available
eigenoptions.}. Transitions into goal states yield a reward of $1$, whereas all
other transitions yield $0$ reward. 

Both the standard and hierarchical SR features enable more sample-efficient
transfer relative to the one-hot state encoding baseline (Figure~\ref{fig:
four_room_row_features_main}b, c). Consistent with the policy-dependence limitation
of standard SRs, that the policy re-optimisation invalidates the established
predictive representations, we observe that agents utilising row features of
random-walk SR ($\text{RW-SR}_{\text{row}}$) suffer from significant performance
degradation upon goal-switching (Figure~\ref{fig: four_room_row_features_main}b,
c). In contrast, agents utilising the HSR row features exhibit strong few-shot
transfer, rapidly adapting to the new goal with significantly fewer episodes
upon transfer. Consistent with the visual examination of learning curves and our
intuitions, HSR agents achieve a significantly higher transfer efficiency
(Figure~\ref{fig: four_room_row_features_main}d), where transfer efficiency is
quantified by the ratio of number of training episodes to reach optimal
performance in $G_1$ and $G_2$ tasks, normalised by the same numbers given the
baseline agent
($\frac{N^{\text{HSR/SR}}_{G_2}/N^{\text{Raw}}_{G_2}}{N^{\text{HSR/SR}}_{G_1}/N^{\text{Raw}}_{G_1}}$)..

We attribute this robustness to the stability of the underlying representation,
arising from the temporal extended components in constructing the HSR that
decouples predictive representations from transient variations in low-level
control. Comparing the predictive representation after learning asymptotes in
the two tasks reveals that standard SR matrices undergo drastic reorganisation
in order to conform with the new optimal policy (Figure~\ref{fig:
four_room_row_features_main}e), whereas the HSR matrices are less variable
following policy changes (Figure~\ref{fig: four_room_row_features_main}f).
Quantitative comparison of magnitudes of change confirmed that the relative
change in HSR matrices following policy adaptation was significantly lower than
that of the SR matrix (Figure~\ref{fig: four_room_row_features_main}g). Such
stability translates to elevated capacity in value estimation. We compare the
reconstruction of the optimal value functions of both the current task ($G_1$)
and the unseen task ($G_2$) give established state features after learning in
the $G_1$ task, with varying numbers of bases. We observe that SR and HSR
explain the optimal value function of the current task equally well, but the HSR
basis consistently yields significantly more accurate reconstruction of optimal
value functions in unseen tasks (Figure~\ref{fig:
four_room_row_features_main}h). This suggests that the HSR features span a
subspace that is more robustly aligned with the manifold of potential value
functions across the task distribution. Note that the improved transferability
and stability of HSR features, relative to standard SR features, are not
artefacts of the online learning process. This is further corroborated by
similar empirical results when agents are trained with pre-trained state
representations (Figure~\ref{fig: hsr_eigenoptions_fully_trained}). Moreover,
the transfer efficiency induced by HSR features remains robust across
environments with different geometry (Figure~\ref{fig: multi_mdp_transfer}) and
option classes beyond eigenoptions (Figure~\ref{fig: option_ablation}). 

\textbf{Online Option Construction.} Notably, the HSR framework is not
contingent on the access to a set of pre-existing options, and can incorporate
online-constructed option set (Figure~\ref{fig: four_room_row_features_main}i;
see Appendix~\ref{sec: hsr_online_option} and Algorithm~\ref{alg:
hsr_online_option} for further details) whilst retaining its utility in
supporting stronger transfer efficiency (Figure~\ref{fig:
four_room_row_features_main}j-l).

\begin{figure*}[ht]
  \begin{center}
    \centerline{\includegraphics[width=\linewidth]{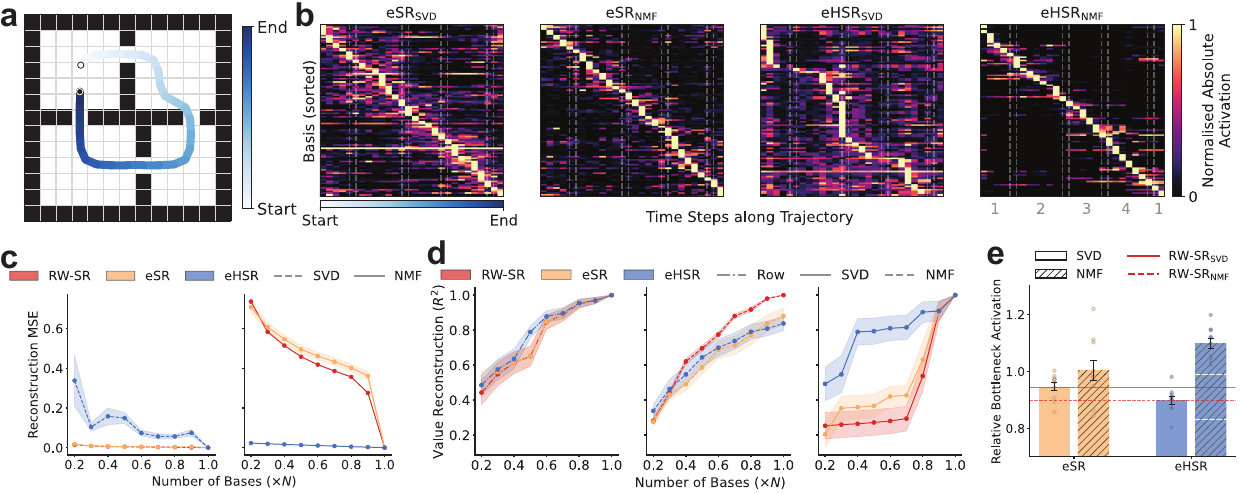}}
    \vspace{-5pt}
    \caption{
    \textbf{HSR-NMF basis yield a sparse, robust, and interpretable
    state representation.}
    \textbf{a.} Example trajectory in the four-room environment.
    \textbf{b.} Activation (normalised) of all basis at each timestep along the
    example trajectory for $\text{eSR}_{\text{SVD}}$, $\text{eSR}_{\text{NMF}}$,
    $\text{HSR}_{\text{SVD}}$, $\text{HSR}_{\text{NMF}}$ (from left to right).
    Gray numbers below the rightmost panel indicates which room the
    corresponding trajectory segment is in.
    \textbf{c.} Reconstruction mean-squared error of predictive representations
    given corresponding low-dimensional features, as functions of varying number
    of bases.
    \textbf{d.} Reconstruction $R^2$ score of optimal value functions with
    respect to randomly selected goal locations, given different state features
    with varying basis size.
    \textbf{e.} Relative bottleneck activation (mean activation at bottleneck
    states / mean activation at non-bottleneck states) of low-dimensional
    features of SR and HSR.
    }
    \label{fig: hsr_nmf_mechanisms_main}
  \end{center}
  \vspace{-25pt}
\end{figure*}

\begin{figure}[ht]
  \centering
  \includegraphics[width=\linewidth]{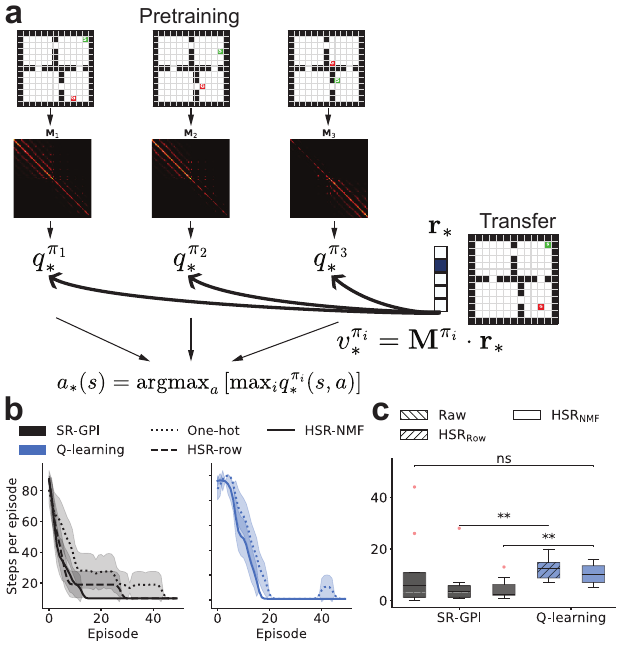}
  \vspace{-15pt}
  \caption{\textbf{HSR-NMF yields comparable transfer efficiency to SR-based
    Generalised Policy Improvement (GPI) algorithm.} 
    \textbf{a.} Schematic of SR-GPI algorithms~\cite{barreto2017successor}.
    \textbf{b.} Number of steps to reach the goal over the first $50$ training
    episodes (mean $\pm$ s.e.; 10 random seeds) of SR-GPI agent and linear
    Q-learning under one-hot and different HSR-based state representations. 
    \textbf{c.} Number of training episodes required to reach optimal
    performance for all agents in \textbf{b} (two-sided two-sample paired
    t-test; SR-GPI with one-hot state representations vs linear Q-learning with
    $\text{HSR}_{\text{NMF}}$ features: $p = 0.199$; SR-GPI with
    $\text{HSR}_{\text{Row}}$ features vs linear Q-learning with
    $\text{HSR}_{\text{Row}}$ features: $p = 0.003$; SR-GPI with
    $\text{HSR}_{\text{NMF}}$ features vs linear Q-learning with
    $\text{HSR}_{\text{NMF}}$ features: $p = 0.007$; $df=18$). }
  \label{fig: hsr_gpi}
  \vspace{-15pt}
\end{figure}

\subsection{HSR-NMF basis facilitates robust transfer.}
We next evaluate the robustness of low-dimensional features derived from the
predictive representations. To ensure fair comparison, we include an additional
baseline which constructs an ``expected SR'' (eSR) given the same set of
pre-training tasks used for constructing the expected HSR (Figure~\ref{fig:
four_room_motivation}a, bottom branch; see also Appendix~\ref{sec:
implementation_detail_transfer} and Algorithm~\ref{alg: sr_td_option}). We
hypothesised that while SVD is optimal for the smooth and diffusive dynamics of
standard SRs, the piecewise-smooth topology of HSR is better captured by NMF
decomposition. 

Conforming with our hypothesis, for RW-SR and eSR, SVD-based features
enables efficient transfer (given pre-training), whereas NMF fails to produce a
useful basis, resulting in poor performance (Figure~\ref{fig:
four_room_nmf_svd_main}a). We attribute the failure of NMF on SR features to
``feature collapse'' given the lack of intra-compartment variability
(Figure~\ref{fig: four_room_motivation}g). Crucially, the pattern is reversed
for HSR: HSR-NMF features not only proved superior to HSR-SVD, but also
outperformed all SR-based baselines, in terms of sample efficiency of transfer
(Figure~\ref{fig: four_room_nmf_svd_main}a, b). Notably, row features of HSR yield more rapid
learning than corresponding SVD features, whereas those of standard SRs
performed poorly. 

We hypothesised that the performance advantage of HSR-NMF stems from the
sparsity and stability of these bases. Indeed, through examining the temporal
trace of basis activation along a circular trajectory around the environment
(Figure~\ref{fig: hsr_nmf_mechanisms_main}a), we confirm that while spectral
decomposition and standard SR features exhibit entangled, noisy activations that
fluctuate rapidly, HSR-NMF reveals a sparse code that uniformly tiles the state
space (Figure~\ref{fig: hsr_nmf_mechanisms_main}b). The localised features
support efficient approximation of arbitrary value function over the state space
(Figure~\ref{fig: hsr_nmf_mechanisms_main}c), whilst retaining the
dynamics-aware predictive nature. In contrast, the globally supported,
oscillatory bases given spectral decomposition yield value learning difficult,
as increasing value estimates in one room likely cause simultaneous decrease in
estimated values in another room, a classical ``whack-a-mole'' issue. The
stability of the basis is underpinned by the intrinsic compressibility of the
HSR matrix. Crucially, we observe that HSR-NMF matches the reconstruction
efficiency of SR-SVD (Figure~\ref{fig: hsr_nmf_mechanisms_main}d), despite its
heavy-tailed spectral distribution (Figure~\ref{fig:
hsr_spectral_supplementary}a). Hence, despite the additional positiveness
penalty, the piecewise smooth structure of the HSR enables NMF to discover
faithful low-dimensional representation whilst achieving SVD-level compression
efficiency. The discovered latents additionally exhibit stronger
interpretability, with elevated activations at bottleneck states
(Figure~\ref{fig: hsr_nmf_mechanisms_main}e), the key states governing
successful navigation underlying arbitrary reward configurations in the state
space.

\begin{figure*}[ht]
  \begin{center}
    \centerline{\includegraphics[width=\linewidth]{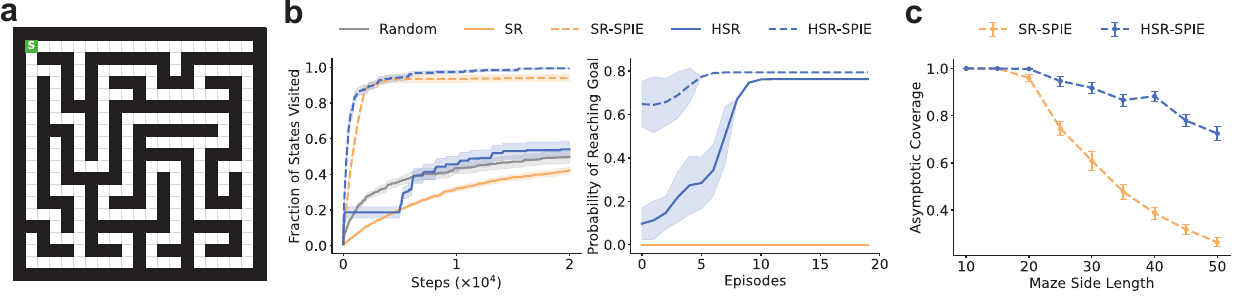}}
    \vspace{-5pt}
    \caption{
      \textbf{Hierarchical temporal abstraction enables scalable intrinsically
      motivated exploration.} 
      \textbf{a.} Exemplar procedurally generated random maze environment.
      \textbf{b.} Learning curves (mean $\pm$ s.e.) of different agents (see
      main text) in terms of pure exploration (in the absence of extrinsic
      reward; left) and goal-directed navigation (with only non-zero reward at
      randomly selected goal locations; right).
      \textbf{c.} Asymptotic state coverage (after $10^5$ interaction steps;
      mean $\pm$ s.e.) for SR-SPIE and HSR-SPIE agents, as a function of maze
      size. 
    }
    \label{fig: hsr_exploration_main}
  \end{center}
  \vspace{-25pt}
\end{figure*}

\vspace{-2pt}
\subsection{HSR and Generalised Policy Improvement}
\vspace{-2pt}
An related class of methods for addressing the transfer efficiency of RL agents
through the lens of representation learning are the Generalised Policy
Improvement (GPI) with SR (Figure~\ref{fig: hsr_gpi}a;
\citealt{barreto2017successor, barreto2018transfer, carvalho2023combining}).
SR-GPI algorithms leverage a similar two-stage process, solve a set of
pre-training tasks (Figure~\ref{fig: four_room_motivation}b) and maintains a
pool of pre-trained policies, each yielding policy-specific SR matrix, $\vec
M^{\pi_g}$. In the downstream transfer task with state-dependent reward vector
$\vec r_\ast$, each pretrained policy is re-evaluated given corresponding SRs
(Equation~\ref{eq: sr_def}), and action selection is then performed via GPI:
$a_\ast(s) = \argmax_a\left[\max_g q_\ast^{\pi_g}(s, a)\right]$. Importantly,
SR-GPI and HSR tackle transfer learning from orthogonal yet complementary
angles, focusing on algorithmic and representational advancements, respectively.

We empirically evaluate the transfer efficiency of SR-GPI, augmented with linear
function approximation given different state representations (Figure~\ref{fig:
hsr_gpi}b,c). We note that SR-GPI requires access to the ground-truth reward
function. To ensure fair comparison, $\vec r_\ast$ is not provided \textit{a
priori}, but needs to be learned online upon first encounter of the single goal
location (see Appendix~\ref{sec: gpi_experiments} for further details). Under
this setup, the canonical SR-GPI with one-hot state representation performed
comparably to the HSR agents that do not perform policy composition. Moreover,
combining SR-GPI with HSR state representations, either row or NMF features,
yields stronger transfer efficiency, hence supporting our argument that HSR and
GPI offer complementary benefits: policy-composition on the one hand, and
generalisable representation on the other.

\subsection{Scalable exploration with HSR.}
We finally investigate whether HSRs facilitate stronger RL beyond providing a
robust state representation. Recent works demonstrate the utility of standard
SRs in constructing effective intrinsic motivation for exploration in
sparse-reward environments~\citep{machado2020count, yu2023successor}. We
specifically consider two SR-based intrinsic rewards: row-norms as a proxy for
state visitation count~\cite{machado2020count}, and the successor-predecessor
intrinsic exploration (SPIE) that combines both temporally prospective and
retrospective information~\cite{yu2023successor}.
\begin{equation}
  r_{\text{SR}}(s) = \frac{1}{||\vec M_{s:}||_1}\,,\quad  r_{\text{SPIE}}(s) = \vec M_{ss'}^2 - ||\vec M_{:s'}||^2\,,
  \label{eq: intrinsic_rewards}
\end{equation}
We hypothesise that the temporally extended structure of HSR could support more
efficient exploration comparing to standard SRs. We hence implement SARSA-based
agents with these intrinsic rewards, computed given either the standard or
hierarchical SR. We test agents in procedurally generated random mazes of
varying sizes (Figure~\ref{fig: hsr_exploration_main}a). 

We firstly evaluate the exploration efficiency of implemented agents --
quantified through the rate at which state coverage increases with experience --
in pure exploration tasks (in the complete absence of extrinsic reward).
Building upon the established superiority of SPIE objectives over SR-norm
objectives~\cite{yu2023successor}, we observe that HSR-augmented objectives
significantly outperform their SR-based counterparts, covering a larger fraction
of the state space within a fixed budget (Figure~\ref{fig:
hsr_exploration_main}b, left). The gap in exploration efficiency becomes
increasingly pronounced as the environment size grows. Whilst the asymptotic
coverage of SR-SPIE degraded drastically in larger mazes, HSR-SPIE maintained
high state coverage (Figure~\ref{fig: hsr_exploration_main}c). In sparse-reward
navigation tasks (only non-zero reward for transitions into randomly selected
goal state), HSR-augmented intrinsic objectives achieved high probability of
discovering the goal location (Figure~\ref{fig: hsr_exploration_main}b, right).
Somewhat surprisingly, despite HSR-norm intrinsic reward yields slow state
coverage, it is nevertheless more effective than both SR-based objectives in
sparse-reward navigation tasks. These results suggest that by incorporating
extended temporal horizon into predictive representations, the HSR enables
agents to ``escape'' from local diffusive barriers, facilitating extended
exploration in environments with complex topology where single-step predictive
methods become inefficient.  

Finally, consistent with the observation that the HSR yields sustained
improvement in transfer efficiency in the absence of pre-constructed option set,
we confirmed empirically that intrinsic rewards facilitated by HSRs learned
given online-constructed options supports more efficient exploration than
SR-based intrinsic rewards (Figure~\ref{fig: hsr_exploration_online_option}).

\vspace{-5pt}

\section{Related Work}
\label{sec: related_works}
\vspace{-3pt}
\textbf{Generalisation in RL.} Recent efforts to improve the generalisability of
RL agents have been predominantly based on algorithmic advances. These
approaches largely focus on enforcing robustness via novel regularisation and
optimisation objectives~\cite{farebrother2018generalization,
cobbe2019quantifying, igl2019generalization}, or by artificially expanding the
diversity of training distributions through data mixing~\cite{wang2020improving}
and augmentation~\cite{laskin2020reinforcement, yarats2021image}. Moreover, an
influential line of works based on generalised policy improvement address
transfer learning through recycling and composing pre-trained
policies~\cite{barreto2017successor, barreto2018transfer,
carvalho2023combining}. Notably, the performance of GPI-based algorithms
critically hinges on the two-stage process involving pre-training on a set of
related tasks, whereas the HSR framework is robust with respect to the existence
of such pre-training phase (Figure~\ref{fig: four_room_row_features_main}).

Alternative to the algorithmic perspective, existing representation learning
methods have typically addressed generalisation through the lens of invariance
with respect to background noise and local perturbations, leveraging techniques
such as bisimulation metrics~\cite{zhang2020learning} and contrastive
learning~\cite{agarwal2021contrastive, laskin2020curl}. However, true multi-task
generalisation requires a task-equivariant representation that captures the
underlying structure of the environment, hence facilitating transfer across
varying task-induced reward configurations. Our HSR framework bridges this gap
by enforcing a temporally abstract, sparse geometry, whilst being robust to
policy shifts. Crucially, the parallel advancements in algorithmic and
representational methods are complementary, and their combination is expected to
result in stronger generalisation.


\textbf{Hierarchical intrinsic exploration.} Solving sparse-reward tasks
requires efficient exploration that extends beyond local dithering. While
existing intrinsic motivations based on prediction
error~\cite{pathak2017curiosity} and state visitation
counts~\cite{bellemare2016unifying, machado2020count} drive effective local
exploration, they often fail to traverse topological bottlenecks and cannot
induce sustained exploration in response to dynamic reward structure, especially
in larger environments~\cite{yu2023successor}. In their seminal work,
\citet{kulkarni2016hierarchical} addressed this by integrating hierarchical RL
with intrinsic motivation, demonstrating that agents reasoning at higher-level
of temporal abstraction could solve hard-exploration tasks where ``flat'' agents
failed. Our work shares this fundamental insight, such that temporal abstraction
is necessary for driving scalable and temporally extended exploration. However,
we approach the problem through the lens of representation learning rather than
explicit hierarchical control. The multi-scale map is built into the
construction of state features, subsequently enables exploration signals to
``jump'' across local diffusive barriers without requiring explicit subgoal
selection.
\vspace{-5pt}

\section{Discussion}
\label{sec: discussion}
\vspace{-3pt}
We introduce hierarchical successor representation, an extension of the
classical SR framework to incorporate temporal abstractions. The temporally
extended predictive representations yield robust state features with respect to
task-induced changes in reward structures and associated policies, hence
facilitating stronger transferability. Low-rank decomposition of HSR features
based on NMF further supports construction of sparse, interpretable state
representations, again leading to improved generalisability. HSR additionally
supports efficient and scalable intrinsically motivated exploration, extending
its advantages beyond the transfer regime. 

We have focused our empirical evaluations on discrete, tabular environments in
order to isolate the topological and representational properties of HSR under
controlled conditions. Extension to continuous MDPs will inevitably incur
deep-learning-based function approximation, and introduce complex confounding
factors, such as non-stationary feature learning~\cite{lyle2019comparative} and
parameter interference~\cite{bengio2013representation}. These confound can
obscure the specific representational characteristics we aim to formalise.
Having established their structural advantages confirmed under tractable
settings in the current work, a natural future direction is to integrate HSR
into deep RL systems, where its stable, policy-agnostic properties could serve
as a robust objective for representation learning in high-dimensional and
continuous domains~\cite{jaderberg2016reinforcement, machado2017laplacian}.

Here we have primarily considered settings in which options are acquired through
pre-exposure to the environment. Consequently, the strongest transfer results
presented here concern tasks with shared transition dynamics. However, HSR also
remains effective when no option set is available a priori, using online option
construction. The procedure used here is deliberately simple and heuristic;
developing principled algorithms for iterative option discovery, and
characterising their theoretical consequences, remains an important direction
for future work. This will be particularly important in non-stationary
environments, where agents must revise their option set as transition structure
changes, allowing the state representation to adapt online while preserving the
generalisation benefits of hierarchical predictive structure.

So far we have interpreted NMF basis of the HSR primarily through the lens of
low-rank decomposition, emphasising data compression. However, a distinct
advantage of NMF is its capability to extract an \textit{over-complete} basis,
which has promising implications in how the neural system performs efficient
coding~\cite{olshausen1997sparse, lewicki2000learning}, particularly when
combined with the predictive map theory of hippocampal neuronal firing through
the lens of successor representation~\cite{stachenfeld2017hippocampus,
yu2020prediction, piray2021linear}. Exploring HSR-NMF in the over-complete
regime presents a promising avenue for scaling to environments with repetitive
or aliased substructures, where spectral bases often fail to discriminate
distinct states.

\clearpage 
\section*{Acknowledgement}
We thank Dominik Straub and anonymous reviewers for helpful discussions and
comments. This work was supported by a Wellcome Trust Investigator Award in
Science (212262/Z/18/Z; M.L.) and the Simons Collaboration on Ecological
Neuroscience (M.L.). Authors declare no conflict of interest.

\section*{Impact Statement}
This paper presents work aiming to advance the field of reinforcement learning
and representation learning. There are many potential societal consequences of
our work, none of which we feel must be specifically highlighted here.

\bibliography{example_paper}
\bibliographystyle{icml2026}

\newpage
\appendix
\onecolumn

\setcounter{figure}{0}
\renewcommand{\thefigure}{S\arabic{figure}}
\setcounter{algorithm}{0}
\renewcommand{\thealgorithm}{S\arabic{algorithm}}
\setcounter{equation}{0}
\renewcommand{\theequation}{S\arabic{equation}}

\section{Derivations and Proofs}

\subsection{Derivation of HSR.}
\label{sec: hsr_derivations}
As a reminder, the HSR is defined as the expected discounted future occupancy
under some high-level policy, $\mu:
\mathcal{S}\times\bar{\mathcal{A}}\rightarrow[0, 1]$, where $\bar{\mathcal{A}} =
\mathcal{A}\cup\Omega$ denotes the augmented action space.


\begin{equation}
  \begin{split}
    \mathcal{M}^\mu_{ss'} &= \mathbb{E}_\mu\left[
      \sum_{t=0}^{\infty}\gamma^t\mathds{1}(s_t, s')|s_0=s
    \right]\\
    &= \sum_{\bar{a}\in\bar{\mathcal{A}}}\mu(\bar{a}|s)\mathbb{E}_{\bar{a}}\left[
        \sum_{t=0}^{\tau_{s\bar{a}}-1}\gamma^t\mathds{1}(s_t, s') + 
        \gamma^{\tau_{s\bar{a}}}\sum_{k=0}^{\infty}\gamma^k\mathds{1}(s_{\tau_{s\bar{a}}+k}, s')|s_0=s, \bar{a}
      \right]\\
    &= \sum_{\bar{a}\in\bar{\mathcal{A}}}\mu(\bar{a}|s)\underbrace{\mathbb{E}_{\bar{a}}\left[
        \sum_{t=0}^{\tau_{s\bar{a}}-1}\gamma^t\mathds{1}(s_t, s')|s_0=s, \bar{a}
      \right]}_{\vec M^{\bar{a}}_{ss'}} + \sum_{\bar{a}\in\bar{\mathcal{A}}}\mu(\bar{a}|s)\sum_{\tilde{s}\in\mathcal{S}}\underbrace{\mathbb{E}_{\bar{a}}\left[
        \gamma^{\tau_{s\bar{a}}}\mathds{1}(s_{\tau_{s\bar{a}}}, \tilde{s}) | s_0=s, \bar{a}
      \right]}_{\vec F^{\bar{a}}_{s\tilde{s}}}\mathcal{M}^\mu_{\tilde{s}s'}\\
    &= \mathcal{B}^\mu_{ss'} + \sum_{\tilde{s}\in\mathcal{S}}\mathcal{G}^\mu_{s\tilde{s}}\mathcal{M}^\mu_{\tilde{s}s'}\,,
  \end{split}
  \label{eq: hsr_def_appendix}
\end{equation}

where we conditioned on the termination state $\tilde{s}$ in the second term of
the penultimate line, then leverage the definition of the \textit{expected
intra-option SR} and \textit{hierarchical continuation kernel}
(Equation~\ref{eq: hsr_bellman}) as following.

\begin{equation}
  \mathcal{B}^\mu_{ss'} = \sum_{\bar{a}\in\bar{\mathcal{A}}}\mu(\bar{a}|s)\vec M^{\bar{a}}_{ss'}\,,\quad 
  \mathcal{G}^\mu_{s\tilde{s}} = \sum_{\bar{a}\in\bar{\mathcal{A}}}\mu(\bar{a}|s)\vec F^{\bar{a}}_{s\tilde{s}}\,,
\end{equation}

where we define $\vec M^{\bar{a}}_{ss'} =
\mathbb{E}_{\pi_{\bar{a}}}\left[\sum_{t=0}^{\tau_{s\bar{a}}-1}\gamma^t\mathds{1}(s_t,
s')|s_0=s\right]$ as the intra-option SR. 

We introduce the \textit{discounted termination kernel}, $\vec{F}^{\bar{a}}$,
defined as following.
\begin{equation}
  \begin{split}
    \vec{F}^{\bar{a}}_{ss'} &= \sum_{t=0}^\infty\gamma^t\mathbb{P}(\bar{a}\text{ terminate at } s' \text{ at time }t | s_0=s)\\
    &= \left(\sum_{t=0}^{\infty}\gamma^t\mathbb{P}^{\bar{a}}(s_{t}=s'|s_0=s)\right)\cdot \beta_{\bar{a} }(s')\\
    &= \left[\vec{M}^{\bar{a}}\cdot\text{diag}(\vec \beta_{\bar{a}})\right]_{ss'}\,,
  \end{split}
\end{equation}

Hence, the second term in the recursive definition (Equation~\ref{eq:
hsr_def_appendix}) can be re-expressed as following.
\begin{equation}
  \mathbb{E}_{\mu, \bar{a}}\left[\gamma^{\tau_{s\bar{a}}}\mathcal{M}^\mu_{s_{\tau_{s\bar{a}}}s'} \right] = 
  \sum_{\tilde{s}\in\mathcal{S}}\mathbb{E}_{\bar{a}}\left[\gamma^{\tau_{s\bar{a}}}\mathds{1}(s_{\tau_{s\bar{a}}}, \tilde{s})|s_0=s\right]\mathcal{M}^\mu_{\tilde{s}s'} = 
  \sum_{\tilde{s}\in\mathcal{S}}\vec F^{\bar{a}}_{s\tilde{s}}\mathcal{M}^\mu_{\tilde{s}s'}\,,
\end{equation}

We hence have derived the Bellman recursion for HSR.

\begin{equation}
  \begin{split}
    &\mathcal{M}_{ss^{\prime}}^{\mu} = \underbrace{\sum_{\bar{a}\in\bar{\mathcal{A}}} \mu(\bar{a}|s) \vec M_{ss^{\prime}}^{\bar{a}}}_{\mathcal{B}^\mu_{ss^{\prime}}} + 
    \sum_{\tilde{s}\in\mathcal{S}} \underbrace{\left( \sum_{\bar{a}\in\bar{\mathcal{A}}} \mu(\bar{a}|s) \vec F_{s\tilde{s}}^{\bar{a}} \right)}_{\mathcal{G}^\mu_{s\tilde{s}}} \mathcal{M}_{\tilde{s}s^{\prime}}^{\mu}\,,\\
    &\mathcal{T}^\mu\mathcal{M} = \mathcal{B}^\mu + \mathcal{G}^\mu\mathcal{M}\,,
  \end{split}
\end{equation}

\subsection{Proof of Theorem~\ref{theorem: hsr_contraction}}
\label{sec: bellman_contraction_proof}
We here provide the proof that the HSR Bellman operator is a contraction mapping
under max-norm.

\begin{proof}
Consider two arbitrary matrices $\vec{M}, \vec{M}'$. From Equation~\ref{eq:
hsr_bellman}, the element-wise difference is:
\begin{equation}
    |(\mathcal{T}^{\mu}\vec{M})_{ss'} - (\mathcal{T}^{\mu}\vec{M}')_{ss'}| = 
    \left| \sum_{\tilde{s}\in\mathcal{S}} \mathcal{G}_{s\tilde{s}}^{\mu} (\vec M_{\tilde{s}s'} - \vec M'_{\tilde{s}s'}) \right| 
    \leq \sum_{\tilde{s}\in\mathcal{S}} \mathcal{G}_{s\tilde{s}}^{\mu} || \vec{M} - \vec{M}' ||_{\infty}
    \leq \left(\sup_s\sum_{\tilde{s}\in\mathcal{S}} \mathcal{G}_{s\tilde{s}}^{\mu}\right) || \vec{M} - \vec{M}' ||_{\infty}
\end{equation}
The supremum over sum term can be re-expressed as
$\left(\sup_s\sum_{\tilde{s}\in\mathcal{S}}
\mathcal{G}_{s\tilde{s}}^{\mu}\right) =
\sum_{\bar{a}}\mu(\bar{a}|s)\gamma^{\tau_{s\bar{a}}} \leq \gamma$, since
$\tau_{s\bar{a}} \geq 1$. Thus, the operator contracts the error by at least
$\gamma$ at each step.
\end{proof}

\section{Further algorithmic details of the HSR framework}

\subsection{Pseudocode for constructing HSR-NMF}

\begin{algorithm}[h]
   \caption{Computing NMF basis of expected HSR (HSR-NMF) given eigenoptions.}
\begin{algorithmic}[1]
   \STATE {\bfseries Input:} MDP $\mathcal{M} = \langle \mathcal{S},
   \mathcal{A}, \mathcal{P}, \mathcal{R}, \gamma \rangle$, set of pre-training
   tasks $G = \{g_1, \dots, g_L\}$, number of options $K$.
   
   \STATE {\bfseries Output:} Basis $\vec\varphi \in \mathbb{R}^{N \times P}$.
   
   \STATE \textbf{// Stage 1: Eigenoption discovery.}
   \STATE Compute RW-SR: $\vec M_{0} \leftarrow (\vec I - \gamma \vec P_{\text{rw}})^{-1}$.
   \STATE Compute Eigenvectors: $\vec U_0, \Sigma_0, \vec V_0 \leftarrow \text{SVD}(\vec M_0)$.
   \STATE Construct Options set $\Omega \leftarrow \emptyset$.
   \STATE Permute column space of $\vec V_0$ given decreasing singular values.
   \FOR{$k=1$ {\bfseries to} $K$}
   \STATE Define pseudo-reward $r_k(s, s') \leftarrow \vec v_k(s') - \vec
   v_k(s)$. 
   \STATE Learn option-specific policy $\pi_k$, initial set, $\mathcal{I}_k$,
   and termination $\beta_k$ by maximising $r_k$ (e.g., with Q-learning).
   \STATE $\Omega \leftarrow \Omega \cup \{( \mathcal{I}_k, \pi_k, \beta_k )\}$.
   \ENDFOR

   \STATE \textbf{// Stage 2: Offline construction of expected HSR.}
   \STATE Initialize Expected HSR $\bar{\mathcal{M}} \leftarrow \mathbf{0}_{N \times N}$.
   \FOR{each task $g \in G$}
   \STATE Learn high-level policy $\mu_g: \mathcal{S} \to \Omega$ for task $g$.
   \STATE Compute \textit{intra-option SR}: $\mathcal{B}_g \leftarrow \sum_{\bar{a}}
   \mu_{g}(\bar{a}|s) \vec M^{\bar{a}}$.
   \STATE Compute \textit{continuation kernel}: $\mathcal{G}_g \leftarrow
   \sum_{\bar{a}} \mu_g(\bar{a}|s) \vec M^{\bar{a}} \text{diag}(\vec \beta_{\bar{a}})$.
   \STATE Solve HSR for task $g$: $\mathcal{M}_g \leftarrow (\vec I -
   \mathcal{G}_g)^{-1} \mathcal{B}_g$.
   \STATE Update Expected HSR: $\bar{\mathcal{M}} \leftarrow \bar{\mathcal{M}} +
   \frac{1}{M} \mathcal{L}_g$.
   \ENDFOR

   \STATE \textbf{// Stage 3: Basis Discovery via NMF}
   \STATE Solve $\min_{\Phi, H} || \bar{\mathcal{M}} - \vec\varphi \vec H
   ||_F^2$ subject to $\vec\varphi, \vec H \geq 0$.
   \STATE \textbf{return} Basis features $\vec\varphi$.
\end{algorithmic}
\label{alg: hsr_nmf}
\end{algorithm}

\subsection{HSR with online-constructed options}
\label{sec: hsr_online_option}
To incorporate online option-construction into the HSR framework, agents are
initialised with empty option sets and identity SR matrix. To construct
eigenoptions, the SR matrix is updated online given primitive actions
(Equation~\ref{eq: sr_td}). Upon reaching a pre-specified number of warm-up
steps, agents compute the set of eigenoptions given eigenvectors of the current
SR matrix (Equation~\ref{eq: eigenoption_reward}), which augments the action
space of the agent and the associated HSR matrix can be learned hereafter
(Equation~\ref{eq: hsr_td}). The option set is updated periodically instead of
over each timestep to stabilise the representation learning process. 

We note that the currently proposed online procedure is largely heuristic, and
we leave the study of alternative procedures and associated theoretical analysis
for future works.

\begin{algorithm}[h]
   \caption{Learning HSR given online-constructed option set.}
\begin{algorithmic}[1]
   \STATE {\bfseries Input:} MDP $\mathcal{M} = \langle \mathcal{S},
   \mathcal{A}, \mathcal{P}, \mathcal{R}, \gamma \rangle$, number of options
   $K$, number of warm-up steps, $T_0$, number of steps for periodically updating
   the option set and associated HSR, $T_{\text{update}}$, maximum number of
   timesteps, $T_{\text{max}}$.
   \STATE {\bfseries Output:} HSR matrix $\mathcal{M} \in
   \mathbb{R}^{N\times N}$.
   \STATE \textbf{Initialisation.} Initialise option set, $\Omega \leftarrow
   \emptyset$; primitive-action SR matrix, $\vec M \leftarrow \vec I_{N\times
   N}$, HSR matrix, $\mathcal{M} \leftarrow \vec I_{N\times N}$, policy $\pi:
   \mathcal{S}\times \mathcal{A}\rightarrow [0, 1]$.
   \FOR{$t=1$ {\bfseries to} $T_0$}
   \STATE Sample $a_t\sim \pi(\cdot|s)$, interact with the environment to retrieve
   the transition tuple $(s_t, a_t, r_t, s_{t+1})$.
   \STATE Update $\vec M$ given the transition tuple.
   \STATE Update policy with standard RL algorithms (e.g., Q-learning, policy
   gradient, etc.).
   \ENDFOR
   \STATE Construct $\Omega$ given the current SR matrix (see lines 6-12 in
   Algorithm~\ref{alg: hsr_nmf}).
   \STATE Augment the action space, $\bar{\mathcal{A}} = \mathcal{A} \cup
   \Omega$.
   \WHILE{$t < T_{\text{max}}$}
   \STATE Sample $\bar{a}_t\sim\pi(\cdot|s)$, interact with the environment to
   retrieve the sequence of transition tuples $\left\{s_i, a_i, r_i,
   s_{i+1}\right\}_{i=t}^{t+\tau_{\bar{a}}}$. 
   \STATE Update HSR (Equation~\ref{eq: hsr_td}) and SR (Equation~\ref{eq:
   sr_td}) given the transition tuple sequence. 
   \IF{$(t - T_0) \mod T_{\text{update}} = 0$}
   \STATE Update $\Omega$ given the current SR matrix.
   \ENDIF
   \STATE $t \leftarrow t + \tau_{\bar{a}}$.
   \ENDWHILE
   \STATE \textbf{return} HSR $\mathcal{M}$.
\end{algorithmic}
\label{alg: hsr_online_option}
\end{algorithm}

\section{Implementation Detail}
\label{sec: implementation_detail}

Python codes for empirical evaluations included in the main paper can be found
at \url{https://github.com/changmin-yu/HSR_icml_2026}.

\begin{algorithm}[ht]
   \caption{Online learning of SR matrix given option-augmented action space.}
\begin{algorithmic}[1]
   \STATE {\bfseries Input:} MDP $\mathcal{M} = \langle \mathcal{S},
   \mathcal{A}, \mathcal{P}, \mathcal{R}, \gamma \rangle$, set of options
   $\Omega = \{(\mathcal{I}_k, \pi_k, \beta_k)\}_{k=1}^{K}$, exploration
   baseline $\epsilon$, time horizon $T_{\text{max}}$.
   
   \STATE {\bfseries Output:} SR matrix $\vec M \in \mathbb{R}^{N \times N}$.
   \STATE Initialise SR matrix, $\vec M = \vec 0_{N\times N}$, augmented action
   value function $Q(s, \bar{a}) = 0$ for all $(s,
   \bar{a})\in\mathcal{S}\times\bar{\mathcal{A}}$, initial state $s_0$.
   \WHILE{$t < T_{\text{max}}$ \AND $Q$ not converged}
   \STATE Selected action $\bar{a}_t$ via $\epsilon$-greedy given $Q$.
   \STATE Execute $\bar{a}_t$ until termination, yielding trajectory $\{s_{t},
   a_{t}, s_{t+1}, \dots, a_{t+\tau_{s_t\bar{a}_t}},
   s_{t+\tau_{s_t\bar{a}_t}+1}\}$.
   \FOR{$\tau=0$ {\bfseries to} $\tau_{s_t\bar{a}_t}$}
   \STATE Given single-step transition tuple $(s_{t+\tau}, a_{t+\tau},
   r_{t+\tau}, s_{t+\tau+1})$
   \STATE Update $\vec M_{s_{t+\tau}:}$ via TD-learning (Equation~\ref{eq:
   sr_td}).
   \STATE Update Q-function with single-step transition tuple, $$Q(s_{t+\tau},
   a_{t+\tau}) \leftarrow Q(s_{t+\tau}, a_{t+\tau}) +
   \alpha\left(r_{t+\tau}+\gamma \max_{\bar{a}\in\bar{\mathcal{A}}}Q(s_{t+\tau+1}, \bar{a})-Q(s_{t+\tau}, a_{t+\tau})\right)\,,$$
   \ENDFOR
   \STATE Increment step counter $t \leftarrow t + \tau_{s_t\bar{a}_t}$.
   \ENDWHILE
   \STATE \textbf{return} SR matrix $\vec M$.
\end{algorithmic}
\label{alg: sr_td_option}
\end{algorithm}

\subsection{Transfer learning}
\label{sec: implementation_detail_transfer}
All implemented agents were implemented based on Q-learning with linear function
approximation. In experiments with simultaneous online learning of value
function and state representation (Figure~\ref{fig:
four_room_row_features_main}), all agents have access to the same set of $8$
eigenoptions (Figure~\ref{fig: hsr_basis_supplementary}a). The same set of
hyperparameters is used for training the option-specific optimal policies,
online/offline construction of SR/HSR, and for learning the optimal high-level
policy in downstream tasks.
\begin{itemize}
  \item Discounting factor: $\gamma = 0.9$;
  \item Learning rate: $\alpha = 0.01$; 
  \item Exploration baseline ($\epsilon$-greedy): $\epsilon = 0.1$;
  \item Number of SVD/NMF basis used: $K = N = 104$;
\end{itemize}

To properly assess the transferability of agents (Figure~\ref{fig:
four_room_row_features_main}, Figure~\ref{fig: hsr_eigenoptions_fully_trained}),
all agents are immediately transferred to learning in the new task upon
goal-switching, without re-initialisation of the linear weighting in function
approximation established over the $G_1$ task. For each task, all agents are
trained over $50$ episodes, with a finite horizon of $5000$ timesteps within
each episode. 

Given the augmented action space, online learning of HSR follows the derived
TD-learning rule (Equation~\ref{eq: hsr_td}). We update the SR matrix also using
option-induced trajectories, but with intra-option step-wise update
(Algorithm~\ref{alg: sr_td_option}).

Sample efficiency of learning (and of transfer) is quantified through the number
of training episodes before (exponential moving average of, with step size
$0.1$) episode-specific steps to reach the goal falls below $1.5\times$ optimal
performance, and maximum number of episodes ($50$) if not converged. 

For transfer learning experiments with offline-constructed predictive features
(Figure~\ref{fig: four_room_nmf_svd_main}, \ref{fig: hsr_nmf_mechanisms_main}),
all agents are trained over $200$ episodes, with a finite horizon of $1000$
timesteps within each episode.

\subsection{HSR and Generalised Policy Improvement}
\label{sec: gpi_experiments}

The SR-GPI algorithms, by design, require knowledge of the ground-truth reward
vector for downstream tasks before any interaction with them. To ensure
maximally fair comparison with the HSR agents, SR-GPI agents are initialised
with linear Q-learning, and switch to the GPI policy (given max-composition of
pretrained policies) once the state-dependent reward vector is constructed. Note
that for simplicity, we assume that there is always one and only one goal
location with non-zero reward, hence SR-GPI agents switch to GPI policies upon
first encounter of the designated goal location.

\subsection{Intrinsically motivated exploration}
\label{sec: exploration_supplementary}
We evaluate SR/HSR-based intrinsic exploration bonus on procedurally generated
$W\times L$ grid-world environments (Figure~\ref{fig: hsr_exploration_main}a).
To ensure topological diversity while ensuring existence of valid solutions, we
implement a custom maze generator based on the recursive backtracking algorithm
(following existing implementations from the \textit{Minigrid} library;
\citealt{chevalier2023minigrid}). The generation process initiates with a grid
fully occluded by walls. A randomised depth-first search algorithm is performed
on the odd-numbered grid-coordinates: starting from a seed-specific randomly
selected cell, the algorithm recursively visits unvisited neighbours, carving
the intervening wall to create passages. 

The reward agents observe at each timestep is a linear summation of extrinsic
and intrinsic rewards.
\begin{equation}
  \tilde{r}(s_t) = \mathcal{R}(s_t) + \lambda r_{\text{intrinsic}}(s_t, s_{t+1})\,,
\end{equation}
where we use SR/HSR-based intrinsic rewards defined in Equation~\ref{eq:
intrinsic_rewards}~\cite{machado2020count, yu2023successor}. 

The scaling factor, $\lambda$ is set to $1$ for all agents in the pure
exploration task (where extrinsic rewards are always $0$; Figure~\ref{fig:
hsr_exploration_main}b, left; Figure~\ref{fig: hsr_exploration_main}c), and to
$0.01$ in the sparse-reward goal-oriented navigation setting (extrinsic reward
equals $1$ for all transitions into the goal state, and $0$ otherwise;
Figure~\ref{fig: hsr_exploration_main}b, right). All agents were implemented as
canonical SARSA agents~\cite{sutton1998reinforcement}, with $\gamma = 0.99$ and
$\alpha = 0.01$, for both types of tasks.

\clearpage
\begin{figure}[ht]
  \begin{center}
    \centerline{\includegraphics[width=\linewidth]{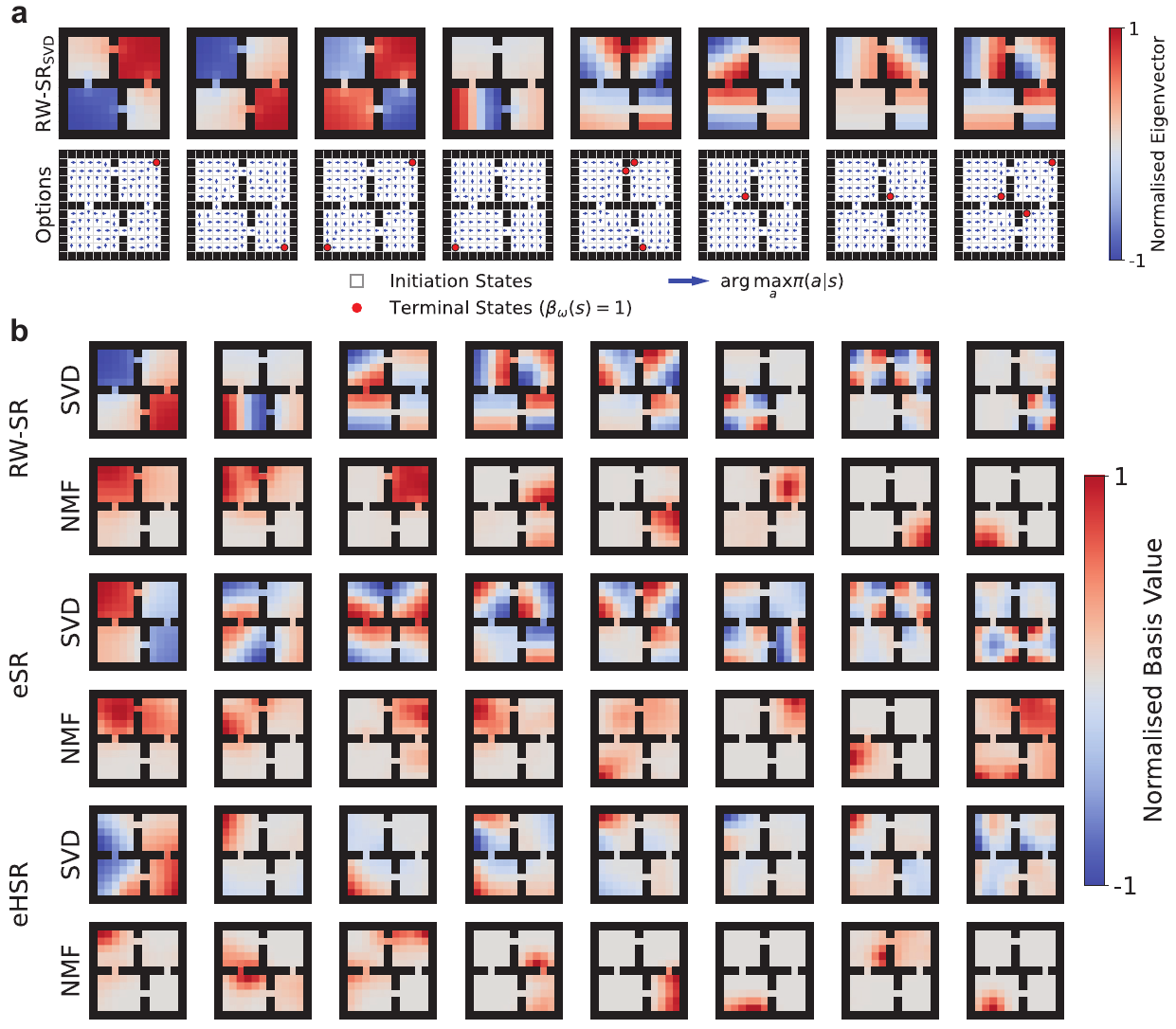}}
    \caption{
      \textbf{Further empirical results on HSR-NMF.} 
      \textbf{a.} The full set of ($8$) eigenoptions used in main experiments,
      both for the option-based RL agents (Figure~\ref{fig:
      four_room_row_features_main}) and for offline construction of expected
      SR/HSR (Figure~\ref{fig: four_room_nmf_svd_main}, \ref{fig:
      hsr_nmf_mechanisms_main}). 
      \textbf{b.} Exemplar features given different decompositions (SVD and NMF)
      of RW-SR, expected SR, and expected HSR. Notably, we observe sparse
      HSR-NMF basis with over-representation of bottleneck states (bottom row).
    }
    \label{fig: hsr_basis_supplementary}
  \end{center}
\end{figure}

\clearpage
\begin{figure}[!t]
  \begin{center}
    \centerline{\includegraphics[width=\linewidth]{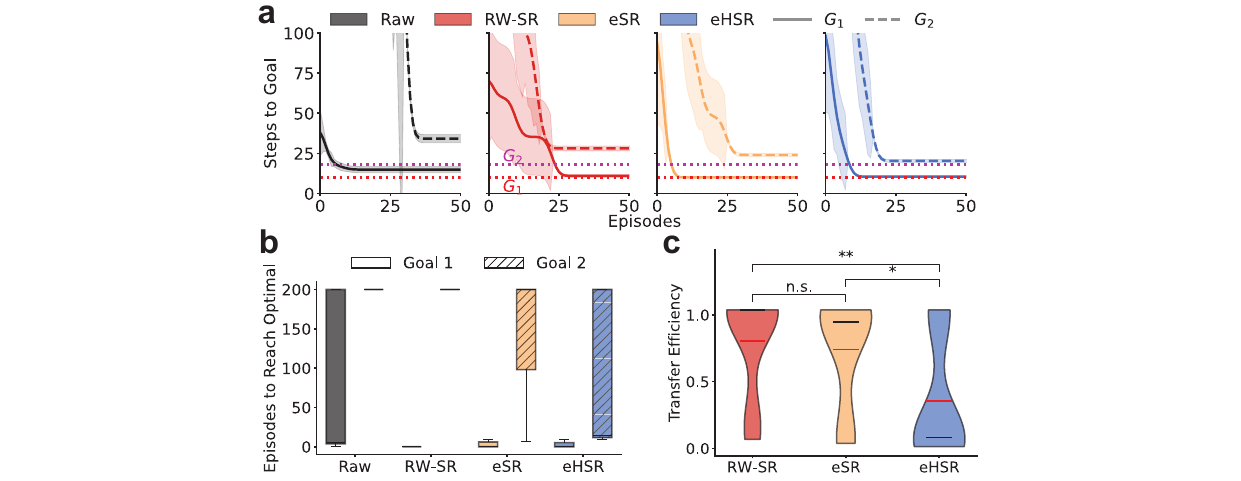}}
    \caption{
      \textbf{Offline-constructed HSR facilitates more sample-efficient transfer in eigenoption-augmented agents.} 
      \textbf{a.} Training curves (number of steps before reaching the goal;
      mean $\pm$ s.e.) of Q-learning agents with linear function approximation
      given row features of different offline-constructed predictive
      representations. All agents were first trained to reach $G_1$, then
      immediately transfered to learn to reach $G_2$.
      \textbf{b.} Number of training episodes to reach optimal performance for
      all agents in \textbf{a}.
      \textbf{c.} Transfer efficiency between $G_1$ and $G_2$ tasks for all agents.
    }
    \label{fig: hsr_eigenoptions_fully_trained}
  \end{center}
\end{figure}

\clearpage
\begin{figure}[!t]
  \begin{center}
    \centerline{\includegraphics[width=\linewidth]{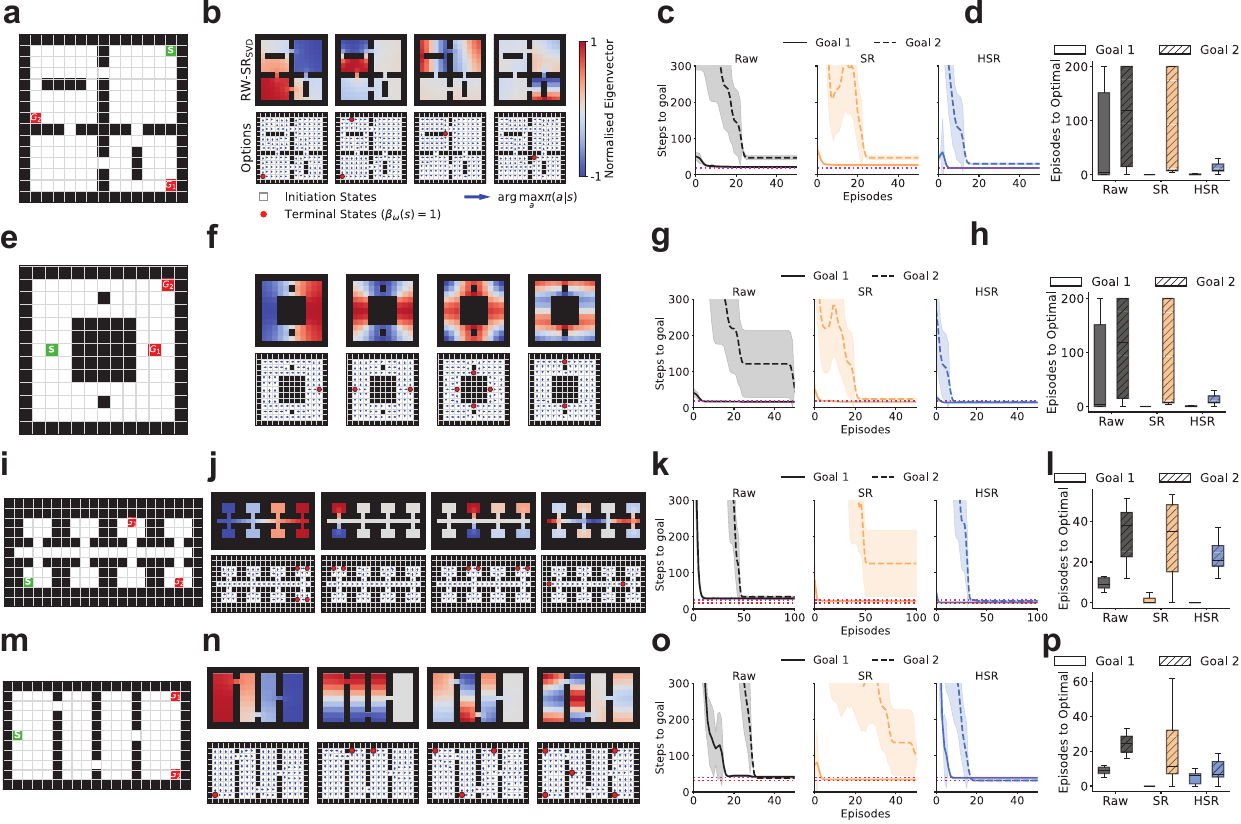}}
    \caption{
      \textbf{HSR transfer-gain is persistent across multiple
      grid-world topologies beyond the canonical four-room environment.}
      \textbf{a.} Graphical demonstration of the augmented four-room environment
      with inserted barriers (breaking the intra-room symmetry), and location of
      the start (S) and the goal locations for the source ($G_1$) and transfer
      ($G_2$) tasks. \textbf{b.} Eigenvectors and option-specific policies for
      selected eigenoptions in the augmented four-room environment (cf. Figure 1d
      in the main text) with increasing spatial frequencies (decreasing
      eigenvalues). \textbf{c.} Training curves (number of primitive actions to
      reach the goal location as a function of number of training episodes; mean
      $\pm$ s.e.; 10 random seeds) for Q-learning agents with linear function
      approximation, given different state representations (left: one-hot
      representation; middle: rows of SR matrix; right: rows of HSR matrix). Note
      that SR- and HSR-based state representations are also trained online (cf.
      Figure 2b in the main text). Dashed horizontal red and magenta lines
      indicate the optimal number of primitive actions to reach $G_1$ and $G_2$
      from $S$, respectively. \textbf{d.} Number of training episodes to reach
      optimal performance in the $G_1$ and $G_2$ tasks for all agents.
      \textbf{e-h.} Same as panels \textbf{a-d}, but for the ``Loop'' maze.
      \textbf{i-l.} Same as panels \textbf{a-d}, but for the ``Corridor'' maze.
      \textbf{m-p.} Same as panels \textbf{a-d}, but for the
      ``Multi-pairwise-bottleneck'' maze. Across all tested grid-world
      environments, HSR retains the strongest transfer performance.
    }
    \label{fig: multi_mdp_transfer}
  \end{center}
\end{figure}

\clearpage
\begin{figure}[!t]
  \begin{center}
    \centerline{\includegraphics[width=\linewidth]{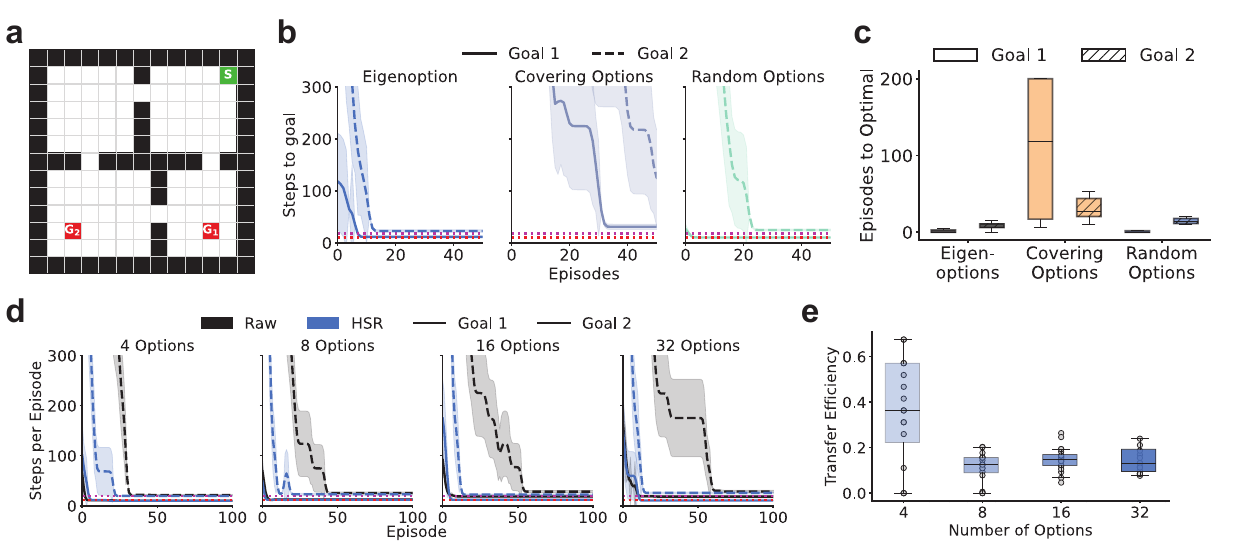}}
    \caption{ \textbf{HSR transfer-gain remains robust across option classes and
    largely stabilises beyond a modest option budget.} \textbf{a.} Graphical
    demonstration of the transfer learning problem setup in the four-room
    environment (cf. Figure 2a in the main text). \textbf{b.} Training curves
    (number of primitive actions to reach the goal location as a function of
    number of training episodes; mean $\pm$ s.e.; 10 random seeds) for
    Q-learning agents with linear function approximation, given online-updated
    HSR state representations. In addition to eigenoptions (left), we also
    implement agents with covering options (middle;
    \citealt{jinnai2019discovering}) and randomly generated target-oriented
    options (given randomly sampled option-specific target locations; right).
    HSR remains effective across alternative option classes, with eigenoptions
    performing best overall, random options slightly worse, and covering options
    substantially weaker in this transfer setting. Covering options performed
    poorly in our transfer setting because they are sparse-reward point options
    with singleton initiation sets, and are therefore only available in very
    limited parts of the state space. This makes them much harder to exploit as
    reusable transfer primitives than eigenoptions, which are broadly available
    and learned from shaped pseudo-rewards. This interpretation is consistent
    with prior analysis showing that covering options are rarely sampled online
    and can suffer exploration difficulties during policy learning (see also
    Machado et al., 2023). \textbf{c.} Number of training episodes requires to
    reach optimal performance in the source and target tasks for HSR agents
    given the three option classes. \textbf{d.} Learning curves for Raw (one-hot
    state representation) and HSR agents given different number of eigenoptions
    (4, 8, 16, or 32). Increasing the number of options enlarges the augmented
    action space and markedly impairs the performance of the Raw baseline,
    whereas HSR remains comparatively robust once a modest option budget is
    available. \textbf{e.} Normalised transfer efficiency (cf. Figure 2d in the
    main text) for HSR agents across option budgets. Transfer efficiency
    improves sharply by increasing the number of options from 4 to 8, and then
    remains broadly stable from 8 to 32 options. This suggests that beyond a
    small minimum set of useful options, HSR is robust to the number of
    incorporated eigenoptions. }
    \label{fig: option_ablation}
  \end{center}
\end{figure}

\clearpage
\begin{figure}[!t]
  \begin{center}
    \centerline{\includegraphics[width=\linewidth]{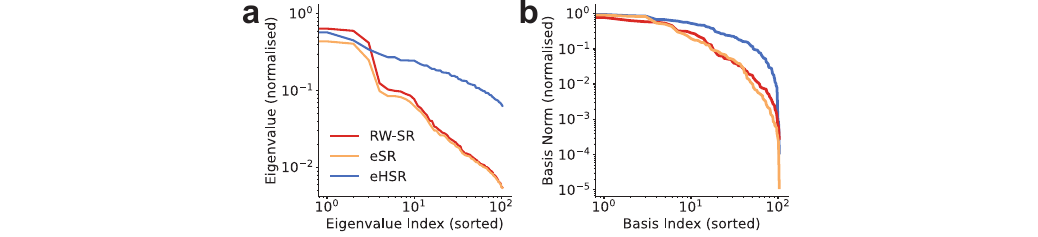}}
    \caption{
      \textbf{Low-dimensional basis of HSR is heavy-tailed distributed.} 
      \textbf{a.} Distribution of eigenvalues of RW-SR, eSR, and eHSR.
      \textbf{b.} Distribution $L_2$-norm of NMF basis of RW-SR, eSR, and eHSR.
    }
    \label{fig: hsr_spectral_supplementary}
  \end{center}
\end{figure}

\clearpage
\begin{figure}[!t]
  \centering
  \includegraphics[width=\linewidth]{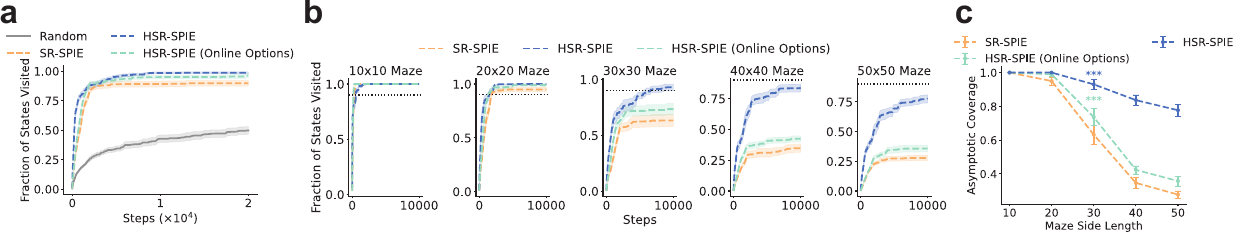}
  \caption{\textbf{HSR with online-constructed options facilitates stronger
  exploration.} \textbf{a.} Learning curves (mean $\pm$ s.e. over 10 random
  seeds) of different agents under pure exploration setting (in the absence of
  extrinsic reward). In addition to the agents considered in the main text (cf.
  Figure~\ref{fig: hsr_exploration_main}), we additionally implement an agent
  with HSR-based SPIE intrinsic reward, given online-constructed options (see
  Appendix~\ref{sec: hsr_online_option} and Algorithm~\ref{alg:
  hsr_online_option}). The HSR-SPIE agent with online-constructed options yields
  comparable exploration efficiency as the HSR-SPIE agent with pre-configured
  eigenoptions, in the $22\times 22$ maze (Figure~\ref{fig:
  hsr_exploration_main}a). \textbf{b.} Same as \textbf{a}, but for mazes with
  increasing sizes. \textbf{c.} Asymptotic coverage (after $10,000$ primitive
  actions) of different agents under increasing maze sizes. Both HSR-SPIE agents
  with pre-computed eigenoptions and online-constructed options consistently
  yield stronger exploration efficiency over SR-SPIE agent across different maze
  sizes (two-way ANOVA test examining the effects of algorithm choice and maze
  size on asymptotic coverage, quantified through the interaction term between
  algorithm choice and maze size; HSR-SPIE vs. SR-SPIE:
  $p=8.7071\times1-^{-18}$, $F_4 = 35.5288$; HSR-SPIE with online-constructed
  options vs. SR-SPIE: $p=1.4480\times10^{-12}$, $F_4 = 21.7160$). Overall, we
  found that the HSR-SPIE intrinsic reward, even under online-constructed
  eigenoptions, consistently outperform SR-SPIE~\cite{yu2023successor}. Note
  that despite the performance gain of HSR-SPIE (relative to SR-SPIE) given
  online-constructed eigenoptions is less salient comparing to its counterpart
  with offline-constructed eigenoptions, it does not rely on the two-stage
  process that requires robust pre-exploration stage, hence is more viable in
  practice.}
  \label{fig: hsr_exploration_online_option}
\end{figure}

\end{document}